\documentclass{article}

\usepackage[preprint]{neurips_2024}

\usepackage[utf8]{inputenc} 
\usepackage[T1]{fontenc}    
\usepackage{hyperref}       
\usepackage{url}            
\usepackage{booktabs}       
\usepackage{amsfonts}       
\usepackage{nicefrac}       
\usepackage{microtype}      
\usepackage{xcolor}         
\usepackage{algorithm} 
\usepackage{algorithmicx} 

\usepackage{microtype}
\usepackage{graphicx}
\usepackage{subfigure}
\usepackage{booktabs} 
\usepackage{caption}
\usepackage{multirow}
\usepackage{hyperref}
\usepackage{array}
\usepackage{amsmath,amssymb,amsthm}

\theoremstyle{plain}
\newtheorem{theorem}{Theorem}[section]
\newtheorem{proposition}[theorem]{Proposition}

\theoremstyle{definition}
\newtheorem{definition}[theorem]{Definition}

\theoremstyle{remark}

\title{Beyond Prior Limits: Addressing Distribution Misalignment in Particle Filtering}

\author{%
  Yiwei Shi$^1$
  \thanks{}, Jingyu Hu $^1$, Yu Zhang$^2$, Mengyue Yang$^1$, Weinan Zhang$^3$, Cunjia Liu$^4$, Weiru Liu$^1$ \\
  $^1$ School of Engineering Mathematics and Technology, University of Bristol\\ 
   $^2$ School of Electronic and Information Engineering, Tongji University\\
  $^3$ Department of Computer Science \& Engineering, Shanghai Jiao Tong University\\
 $^4$ Department of Aeronautical and Automotive Engineering, Loughborough University\\
\texttt{\{yiwei.shi,jingyu.hu,mengyue.yang,weiru.liu\}@bristol.ac.uk} \\
\texttt{yu.zhang@tongji.edu.cn, c.liu5@lboro.ac.uk, wnzhang@sjtu.edu.cn} \\
}

\begin{document}

\maketitle

\begin{abstract}
Particle filtering is a Bayesian inference method and a fundamental tool in state estimation for dynamic systems, but its effectiveness is often limited by the constraints of the initial prior distribution, a phenomenon we define as the Prior Boundary Phenomenon. This challenge arises when target states lie outside the prior’s support, rendering traditional particle filtering methods inadequate for accurate estimation. Although techniques like unbounded priors and larger particle sets have been proposed, they remain computationally prohibitive and lack adaptability in dynamic scenarios. To systematically overcome these limitations, we propose the Diffusion-Enhanced Particle Filtering Framework, which introduces three key innovations: adaptive diffusion through exploratory particles, entropy-driven regularisation to prevent weight collapse, and kernel-based perturbations for dynamic support expansion. These mechanisms collectively enable particle filtering to explore beyond prior boundaries, ensuring robust state estimation for out-of-boundary targets. Theoretical analysis and extensive experiments validate framework's effectiveness, indicating significant improvements in success rates and estimation accuracy across high-dimensional and non-convex scenarios. 
\end{abstract}

\section{Introduction}

Particle filtering has become an essential tool in state estimation for dynamic systems, with widespread applications in fields such as target tracking \cite{djuric2008target}, robotics \cite{thrun2002particle}, and sensor data fusion \cite{caron2007particle}. At its core, particle filtering operates by maintaining a set of weighted particles that approximate the posterior distribution of the system’s state. Despite its effectiveness, particle filtering faces significant limitations due to the dependence on the initial prior distribution. In this study, we introduce the \textit{Prior Boundary Phenomenon (PBP)}, a term we define to describe how the particles’ support range is restricted to the region defined by the prior, rendering states outside this boundary inaccessible. To the best of our knowledge, this study is one of the first to systematically identify and analyse this phenomenon within the context of particle filtering

\begin{figure}[!t]
\centering
\subfigure[Prior Boundary Phenomenon in Particle filtering]{
  \includegraphics[width=0.9\linewidth,height=0.17\textheight]{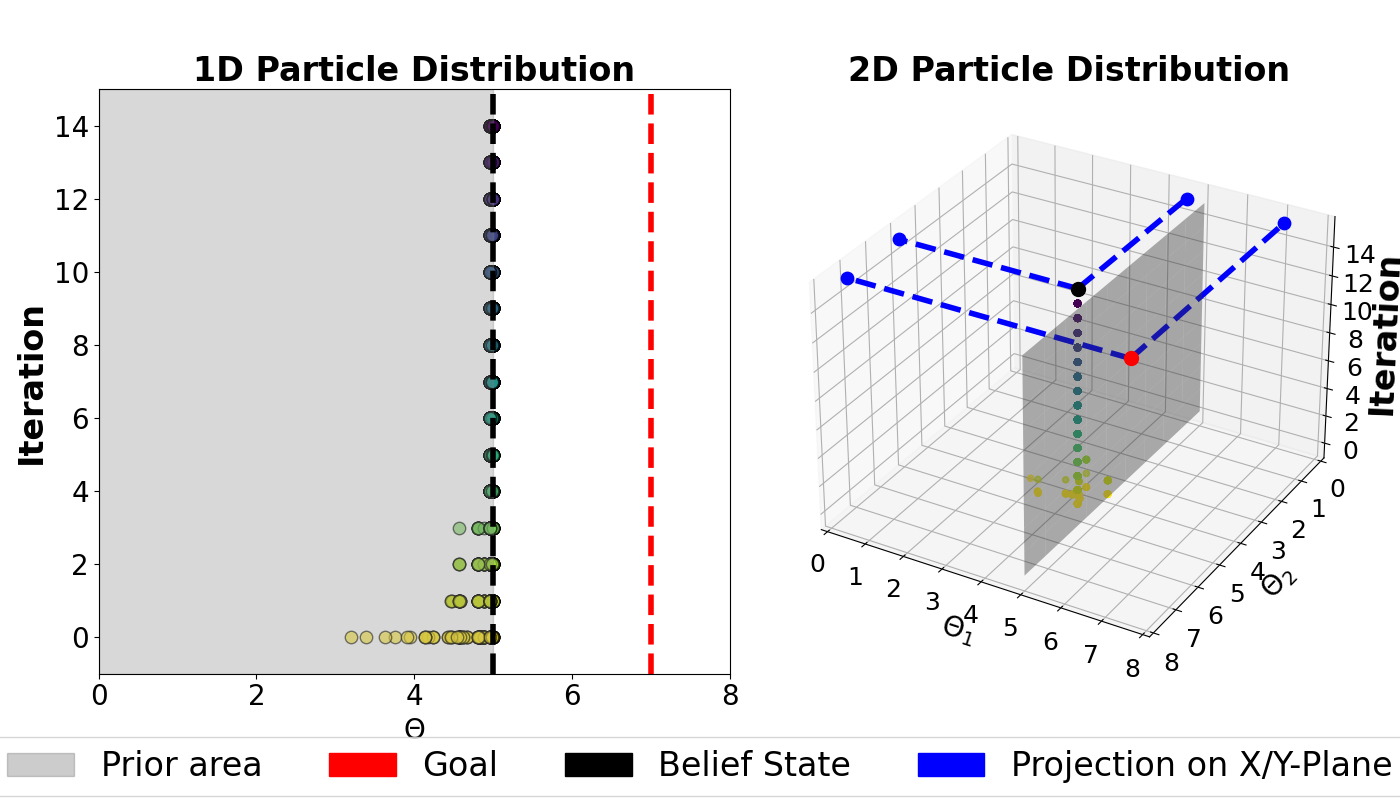}
  \label{fig:1Dand2D}}

\subfigure[Breaking the Prior Boundary in Particle Filtering]{\includegraphics[width=0.9\linewidth]{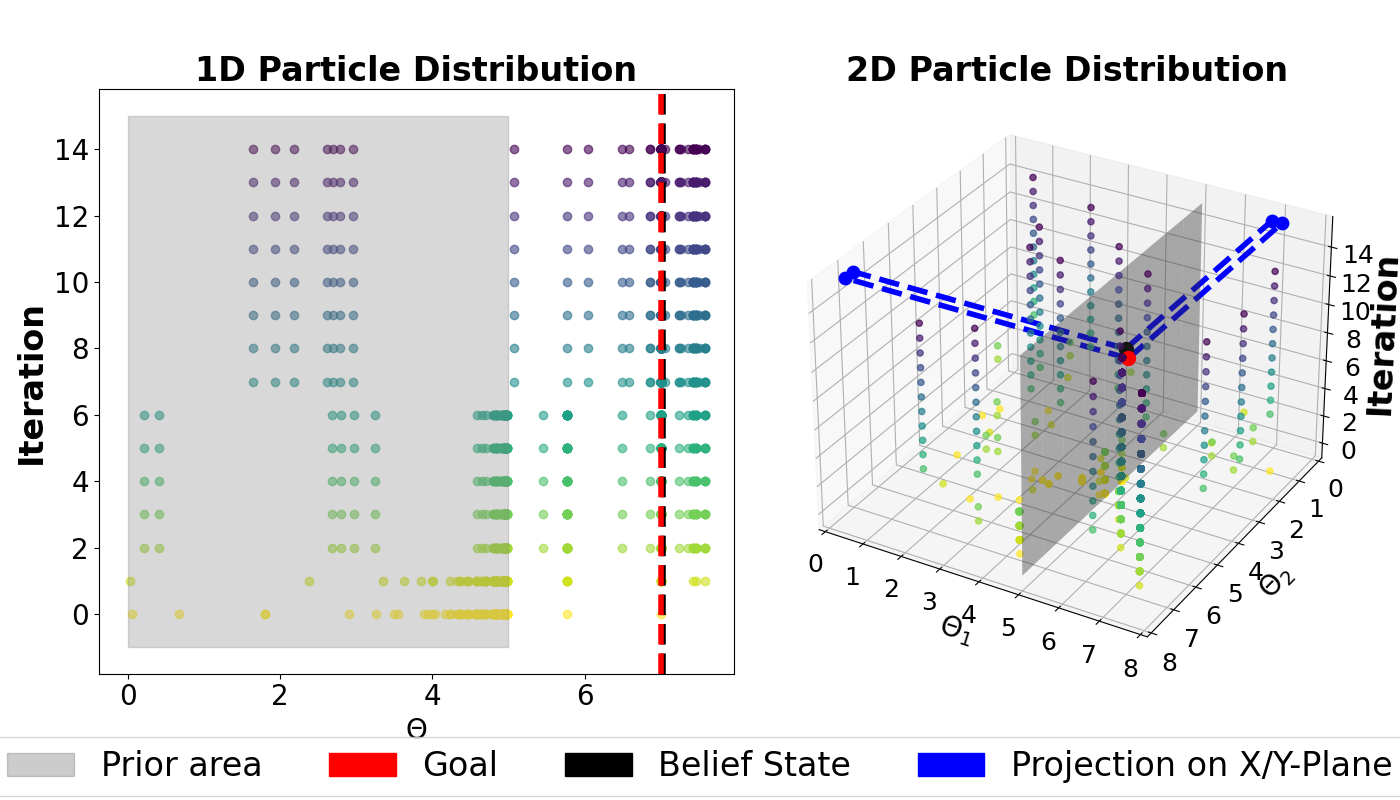}
  \label{fig:1Dand2D_mcmc}}
  \captionsetup{skip=5pt} 
\caption{Visualising the Prior Boundary Phenomenon: 1D and 2D}
\label{fig:1Dand2DVisualisation}
\end{figure}
\setlength{\textfloatsep}{5pt}

This limitation becomes particularly evident when considering the sample efficiency of particle filtering, which heavily depends on the initial prior distribution. The Prior Boundary Phenomenon confines particles to the region defined by the prior, significantly restricting their ability to explore and estimate states beyond this boundary. Such constraints are especially critical for tasks where the target state lies outside the support of the prior. To illustrate the impact of the Prior Boundary Phenomenon, we provide two examples. In the one-dimensional scenario depicted on the left of Fig. \ref{fig:1Dand2D}, a robot attempts to identify its position along a linear path \cite{doucet2001introduction,doucet2001sequential}. The target region spans the interval [0, 5], while the initial prior distribution is restricted to a smaller interval [4.9,5], reflecting the robot's limited initial belief about its location. In this setup, particles generated by traditional methods often remain confined within the narrow prior, even when the target lies outside this range. This occurs due to limited exploration capability in the dynamic model and insufficient correction from the observation model. Similarly, in the two-dimensional scenario \cite{hutchinson2018information} shown on the right of Fig. \ref{fig:1Dand2D}, a mobile robot is tasked with locating a chemical leak. Here, the target region is much larger than the prior distribution, which initially covers only a small subregion. Despite the robot’s attempts to update its position based on sensor readings, particles fail to escape the prior boundary, leaving critical areas unexplored. In both cases, these failures highlight the restrictive nature of traditional particle filtering methods when operating within constrained priors.

\textcolor{black}{The confinement of Particles within the prior boundary persists regardless of the prior distribution's form, such as \textit{Gaussian}, \textit{Beta}, \textit{Dirichlet}, or \textit{non-convex} shapes. Traditional particle filtering \cite{cappe2007overview,doucet2000sequential,Gordon1993} methods typically assume that the target state aligns with the prior's support, focusing on improvements in particle sampling \cite{bishop2012getting, zhou2016new} and resampling \cite{zhou2016new, kuptametee2022review}.
While these approaches may marginally expand the support, they come with significant computational costs and lack adaptability to dynamic environments or unknown states. Additionally, misalignments between the state transition and observation models further exacerbate these limitations, leaving the algorithms ill-equipped to handle cases where the target lies outside the predefined prior boundaries. These challenges underscore the need for more dynamic and flexible techniques capable of addressing boundary constraints, enhancing exploration, and improving particle filtering in complex and safety-critical scenarios.}

To address these limitations, we propose the \textit{Diffusion-Enhanced Particle Filtering (DEPF) Framework}, which integrates three key mechanisms: \textit{Exploratory Particles for Adaptive Diffusion}: A subset of particles is sampled from an extended state space, enabling exploration beyond the prior boundary and enhancing adaptability to new observations. \textit{Entropy-Driven Regularisation}: An entropy-based term is incorporated during weight updates to prevent weight collapse and encourage exploration of low-probability regions. \textit{Kernel-Based Perturbations for Local Expansion}: Gaussian kernels are used to perturb particle positions, dynamically expanding the support range and improving coverage of the target state space. These mechanisms collectively enable DEPF to adapt to new observations, extend the particle support range, and maintain computational efficiency. Theoretical guarantees for support range expansion are also established, ensuring robust and accurate state estimation.

The main contributions of this paper are as follows:  
1). We formalise the \textit{Prior Boundary Phenomenon}, providing a theoretical framework and rigorous proofs to explain the recursive confinement of particles within the prior boundary. This formalisation also includes mathematical guarantees for overcoming these constraints.  
2). We present the \textit{DEPF Framework}, which integrates adaptive diffusion, entropy-based regularisation, and kernel-based perturbations to systematically address prior boundary limitations.  
3). We validate DEPF through comprehensive experiments across diverse scenarios, including high-dimensional and non-convex state spaces. The results demonstrate its superior performance in overcoming prior boundary constraints and achieving robust state estimation.

\section{Related Work}
\textbf{Prior Region Misalignment in Bayesian Inference:} The effectiveness of particle filtering relies heavily on the alignment between the prior distribution and the true target state. Historically, misaligned priors have led to biased state estimations, a challenge highlighted by foundational works \cite{steel2010bayesian,kruschke2010bayesian}. Traditional techniques, such as importance sampling \cite{ristic2013particle}, operate effectively within prior boundaries but fail for states beyond the initial support. Methods like adaptive particle filters \cite{fox2001kld} and hybrid resampling strategies \cite{douc2005comparison} have sought to address sample efficiency but remain limited in their ability to dynamically adapt. 

\textbf{Bayesian Inference of Prior Boundary Phenomena:} The \textit{Prior Boundary Phenomenon}, formalised in this work, extends insights from particle filtering’s recursive limitations \cite{doucet2000sequential}. Traditional methods, including the bootstrap filter \cite{gordon1993novel}, confine particles within the prior’s initial support. This limitation has been underscored in dynamic contexts such as robot localisation \cite{thrun2002probabilistic} and sensor fusion \cite{candy2016bayesian}, where prior misalignment impacts performance. However, these methods focus on local mismatches and fail to address global misalignment or adaptive boundary expansion. 

\textbf{Bayesian Inference in Out-of-Distribution Problems:} Out-of-distribution (OOD) problems represent a critical challenge for Bayesian inference, as they require models to generalise beyond prior knowledge. Advances such as Posterior Networks \cite{charpentier2020posterior} and Bayesian OOD detection with uncertainty exposure \cite{wangbayesian} have addressed OOD detection in classification tasks but are less effective in sequential estimation. Traditional particle filtering methods, such as Rao-Blackwellised filters \cite{li2004estimation} and model-switching approaches \cite{moradkhani2005dual}, demonstrate limited adaptability to OOD scenarios.

\section{Peliminaries}
\label{sec:Peliminaries}

\subsection{Mathematical Description of Particle Filtering}

Particle filtering is a sequential Bayesian estimation technique designed to approximate the posterior distribution of dynamic systems, $P(x_t | y_{1:t})$, where $x_t$ represents the latent state at time $t$ and $y_{1:t}$ represents the sequence of observations up to time $t$. The method operates by maintaining a set of weighted particles $\Theta_t = \{x_t^{(i)}, w_t^{(i)}\}_{i=1}^N$, which provide a discrete approximation to the posterior distribution:
\[
P(x_t | y_{1:t}) \approx \sum_{i=1}^N w_t^{(i)} \delta(x_t - x_t^{(i)}),
\]
where $\delta(\cdot)$ is the Dirac delta function. The weights $w_t^{(i)}$ satisfy the normalization condition $\sum_{i=1}^N w_t^{(i)} = 1$.

Each particle $x_t^{(i)}$ represents a potential realisation of the system state, and its weight $w_t^{(i)}$ reflects the likelihood of the particle given the observed data $y_{1:t}$. Over time, the particles propagate through the state transition model $P(x_t | x_{t-1})$, while the weights are updated based on the observation model $P(y_t | x_t)$. This process enables particle filtering to dynamically adapt to new observations and approximate the posterior distribution in a computationally efficient manner.

\subsection{Prior Boundary Definition in Particle Filtering}

In particle filtering, the effectiveness of state estimation is fundamentally constrained by the prior distribution, which determines the initial support range of particles. This subsection formalises key concepts relevant to the prior boundary and its role in particle filtering:

\begin{definition}[Prior Boundary]
The prior boundary is the region of the state space where the prior distribution $P(x_0)$ has non-zero probability. Mathematically, it is defined as:
\[
\mathcal{S}_{\text{prior}} = \{x_0 \in \mathcal{X} : P(x_0) > 0\},
\]
where $\mathcal{X}$ denotes the entire state space. The prior boundary $\mathcal{S}_{\text{prior}}$ defines the initial region of interest that particles can explore at time $t = 0$.
\end{definition}

\begin{definition}[Support Range]
The support range at time $t$, denoted as $\mathcal{S}_t$, represents the subset of the state space where the particle distribution has non-zero probability. It evolves recursively based on the state transition model $P(x_t | x_{t-1})$ and the previous support range $\mathcal{S}_{t-1}$. Formally, it is expressed as:
\[
\small
\mathcal{S}_t = \bigcup_{i=1}^N \{x_t \in \mathcal{X} : P(x_t | x_{t-1}^{(i)}) > 0\}, \quad \text{where } x_{t-1}^{(i)} \in \mathcal{S}_{t-1}.
\]
\end{definition}
The recursive evolution of the support range is fundamentally constrained by the prior boundary.

The recursive relationship described in the proposition implies that the initial prior boundary $\mathcal{S}_{\text{prior}}$ imposes an absolute constraint on the regions of the state space that particles can reach during the filtering process. This highlights the critical dependence of particle filtering on the prior distribution and its boundary.

\subsection{Proof and Statement of the Prior Boundary Phenomenon}

Given the definitions of the prior boundary $\mathcal{S}_{\text{prior}}$ and the support range $\mathcal{S}_t$, we formally prove that the recursive nature of particle filtering confines the support range within the prior boundary, thereby limiting its ability to explore regions outside $\mathcal{S}_{\text{prior}}$.

\begin{proposition}[Recursive Confinement of Support Range]
The support range $\mathcal{S}_t$ at time $t$ satisfies the recursive relationship:
\[
\mathcal{S}_t \subseteq \mathcal{S}_{t-1} \subseteq \cdots \subseteq \mathcal{S}_{\text{prior}}, \quad \forall t \geq 0.
\]
with the base case:
\[
\mathcal{S}_0 = \mathcal{S}_{\text{prior}}.
\]
Thus, for any time $t \geq 0$, the support range is strictly confined by the prior boundary:
\[
\mathcal{S}_t \subseteq \mathcal{S}_{\text{prior}}.
\]
\end{proposition}

\begin{proof}
By definition, the support range $\mathcal{S}_t$ is given as:
\[
\mathcal{S}_t = \bigcup_{i=1}^N \{x_t \in \mathcal{X} : P(x_t | x_{t-1}^{(i)}) > 0\}, \quad \text{where } x_{t-1}^{(i)} \in \mathcal{S}_{t-1}.
\]
For the base case at $t = 0$, we have:
\[
\mathcal{S}_0 = \mathcal{S}_{\text{prior}} = \{x_0 \in \mathcal{X} : P(x_0) > 0\}.
\]
By induction, assume $\mathcal{S}_{t-1} \subseteq \mathcal{S}_{\text{prior}}$. Since particle propagation at time $t$ depends entirely on $x_{t-1}^{(i)} \in \mathcal{S}_{t-1}$ and the state transition model $P(x_t | x_{t-1})$, which cannot generate particles outside $\mathcal{S}_{t-1}$, it follows that:
\[
\mathcal{S}_t \subseteq \mathcal{S}_{t-1}.
\]
Combining this with the inductive hypothesis $\mathcal{S}_{t-1} \subseteq \mathcal{S}_{\text{prior}}$, we conclude:
\[
\mathcal{S}_t \subseteq \mathcal{S}_{\text{prior}}, \quad \forall t \geq 0.
\]
\end{proof}

This proposition implies that if a target state $x_g$ satisfies:
\[
x_g \notin \mathcal{S}_{\text{prior}},
\]
then $x_g \notin \mathcal{S}_t$ for all $t \geq 0$. In other words, the particle filter cannot generate particles in the vicinity of $x_g$, and the posterior distribution $P(x_t | y_{1:t})$ cannot approximate the true posterior for states outside the prior boundary. The detailed proof is provided in Appendix~\ref{seb:ProofProposition}.

\begin{definition}[Prior Boundary Phenomenon]
The prior boundary phenomenon states that in particle filtering, the recursive evolution of the support range $\mathcal{S}_t$ is strictly bounded by the initial prior boundary $\mathcal{S}_{\text{prior}}$. If a target state $x_g$ lies outside the prior boundary, i.e., $x_g \notin \mathcal{S}_{\text{prior}}$, the particle filter will converge entirely within $\mathcal{S}_{\text{prior}}$ and fail to explore or estimate states outside this region.
\end{definition}

This phenomenon highlights a fundamental limitation of particle filtering, where the prior distribution completely determines the accessible regions of the state space. The inability to estimate target states outside $\mathcal{S}_{\text{prior}}$ forms the central research problem addressed in this work.

\section{Methodology}
Particle filtering operates as a sequential Bayesian estimation framework, leveraging particles and their associated weights to approximate posterior distributions. While effective in many scenarios, the recursive nature of the particle filtering process inherently limits the particle support range $\mathcal{S}_t$ to the prior boundary $\mathcal{S}_{\text{prior}}$, as defined by the initial prior distribution $P(x_0)$. This limitation, known as the Prior Boundary Phenomenon, arises because particles cannot be generated outside $\mathcal{S}_{\text{prior}}$, restricting the filter's ability to estimate states beyond this range.

To overcome the constraints imposed by the Prior Boundary Phenomenon, this section introduces an enhanced particle filtering framework. The proposed methodology combines adaptive exploration, entropy regularisation, and kernel-based perturbations to expand the effective support range of particles, allowing the particle filter to explore and estimate states outside $\mathcal{S}_{\text{prior}}$. Below, we first describe the standard steps of particle filtering, followed by the proposed modifications to address this phenomenon.

\
\subsection{Particle Filtering: Recursive Approximation of Posterior Distributions}
The particle filtering process can be summarised in four recursive steps: sampling, weight update, normalization, and resampling. These steps enable the particle filter to sequentially approximate the posterior distribution $P(x_t | y_{1:t})$ by adapting the particle set to new observations.

\textbf{Step 1. Sampling  (Prediction):} \cite{gordon1993novel} In this step, particles are drawn from an importance distribution \( q(x_t | x_{t-1}, y_t) \), which reflects the conditional probability of the current state given the previous state and the current observation. Typically, the state transition model \( P(x_t | x_{t-1}) \) is used as an approximation of the importance distribution. The propagation of particles can thus be expressed as:
\[
x_t^{(i)} \sim P(x_t | x_{t-1}^{(i)}),
\]
where \( x_t^{(i)} \) denotes the \( i \)-th particle at time \( t \). This step models the dynamics of the system, ensuring that particles evolve in accordance with the state transition process.

\textbf{Step 2. Weight Update:} \cite{doucet2001introduction} After sampling, particle weights are updated to reflect their relative importance with respect to the posterior distribution. The unnormalized weight of a particle is computed using the importance weight formula:
\[
\tilde{w}_t^{(i)} = w_{t-1}^{(i)} \cdot \frac{P(y_t | x_t^{(i)}) P(x_t^{(i)} | x_{t-1}^{(i)})}{q(x_t^{(i)} | x_{t-1}^{(i)}, y_t)},
\]
where \( P(y_t | x_t^{(i)}) \) is the observation model, and \( q(x_t | x_{t-1}, y_t) \) represents the importance distribution used during the sampling step.

\textbf{Step 3. Normalization of Weights:} \cite{liu1998sequential} The weights are then normalized to ensure they sum to 1, maintaining their probabilistic interpretation:
\[
w_t^{(i)} = \frac{\tilde{w}_t^{(i)}}{\sum_{j=1}^N \tilde{w}_t^{(j)}}.
\]
This normalization allows the particle filter to accurately represent the posterior distribution.

\textbf{Step 4. Resampling:} \cite{doucet2000sequential} Over time, particle weights tend to suffer from the \textit{weight degeneracy problem}, where a small number of particles dominate the weight distribution. To address this, a resampling step is performed when the effective number of particles:
\[
N_{\text{eff}} = \frac{1}{\sum_{i=1}^N (w_t^{(i)})^2}
\]
falls below a predefined threshold. Resampling generates a new particle set by sampling from the current set according to their weights. The weights of the resampled particles are then reset uniformly:
\[
w_t^{(i)} = \frac{1}{N}.
\]

These four steps operate recursively to approximate the posterior distribution \( P(x_t | y_{1:t}) \) while adaptively updating the particle set to incorporate new observations. However, the recursive nature of this process confines the particle set to regions defined by the prior boundary \( \mathcal{S}_{\text{prior}} \), limiting the filter’s ability to estimate states outside this range. This limitation, known as the Prior Boundary Phenomenon, necessitates methodological enhancements, which are proposed in the following subsection.

\subsection{Diffusion-Driven Support Range Expansion}

The Prior Boundary Phenomenon arises due to the recursive confinement of the particle support range \( \mathcal{S}_t \) within the initial prior boundary \( \mathcal{S}_{\text{prior}} \). If a target state \( x_g \notin \mathcal{S}_{\text{prior}} \), no particles can be generated near \( x_g \), resulting in the failure of state estimation. To address this, we propose a diffusion-driven particle filtering framework that incorporates dynamic exploration, entropy-driven diffusion regularisation, and kernel-based stochastic perturbations to expand the effective support range.

\textbf{Adaptive Diffusion through Exploratory Particles}\\
At each time step, a subset of particles is designated as \textit{exploratory particles}, which introduce a uniform diffusion process into the framework. These particles are sampled from an extended bounding box \(\mathcal{B}\) that covers regions beyond \( \mathcal{S}_{\text{prior}} \):
\[
x_t^{(j)} \sim \mathcal{U}(\mathcal{B}), \quad j \in \mathcal{E},
\]
where \( \mathcal{B} \) defines the extended state space, and \( \mathcal{E} \) represents the indices of exploratory particles. The exploratory particles are initialised with small weights:
\[
w_t^{(j)} = \frac{\epsilon}{|\mathcal{E}|}, \quad \epsilon \ll 1.
\]
This diffusion mechanism enables the particle filter to sample states outside the original support range \( \mathcal{S}_{\text{prior}} \), thereby increasing the likelihood of reaching states such as \( x_g \notin \mathcal{S}_{\text{prior}} \).

\textbf{Entropy-Driven Diffusion Regularisation}\\
To ensure that the exploratory diffusion does not collapse prematurely, an entropy regularisation term is added during the weight update step. This regularisation diffuses the weights across all particles, encouraging exploration of low-probability regions:
\[
w_t^{(i)} = w_t^{(i)} + \beta H,
\]
where \( H \) is the entropy of the weight distribution, defined as:
\[
H = -\sum_{i=1}^N w_t^{(i)} \log(w_t^{(i)} + \epsilon).
\]
By penalising weight distributions that become overly concentrated, this mechanism promotes balanced diffusion across the state space. The diffusion of weights helps exploratory particles retain influence and encourages the discovery of regions beyond \( \mathcal{S}_{\text{prior}} \).

\textbf{Kernel-Induced Stochastic Diffusion}\\
To further expand the particle support range dynamically, we introduce a stochastic diffusion mechanism based on kernel perturbations. Each particle \( x_t^{(i)} \) is perturbed by a Gaussian kernel that models diffusion within the local neighbourhood:
\[
\Delta x_t^{(i)} \sim h_{\text{opt}} \cdot \mathcal{L} \cdot \mathcal{N}(0, I),
\]
where:
- \( h_{\text{opt}} = A \cdot N^{-\frac{1}{n+4}} \) is the optimal kernel bandwidth dynamically adjusted to balance exploration and precision;
- \( \mathcal{L} \) is the lower triangular matrix obtained from the Cholesky decomposition of the covariance matrix \( \Sigma \), ensuring diffusion adapts to the local particle distribution.

The covariance matrix \( \Sigma \) is computed dynamically:
\[
\Sigma = \sum_{i=1}^N w_t^{(i)} (x_t^{(i)} - \mu)(x_t^{(i)} - \mu)^T + \lambda I,
\]
where \( \mu = \sum_{i=1}^N w_t^{(i)} x_t^{(i)} \) is the weighted mean, and \( \lambda > 0 \) ensures positive definiteness of \( \Sigma \).

This perturbation mechanism expands the effective support range by introducing stochastic diffusion, allowing particles to explore new regions iteratively:
\[
x_t^{(i)} \leftarrow x_t^{(i)} + \Delta x_t^{(i)}.
\]

\textbf{Diffusion-Driven Validation via MCMC}\\
To ensure consistency with the target posterior distribution, a Metropolis-Hastings (MCMC) acceptance criterion \cite{hastings1970monte} validates the diffused particles. For each perturbed particle \( x_t^{(i)} \), the acceptance probability is:
\[
\alpha_i = \frac{w_{\text{new}}^{(i)}}{w_{\text{old}}^{(i)}} \cdot \exp\left(-\frac{1}{2} \Delta x_t^{(i)^T} \Sigma^{-1} \Delta x_t^{(i)}\right).
\]
A uniformly sampled random variable \( u_i \sim \mathcal{U}(0, 1) \) determines whether the particle is accepted:
\[
x_t^{(i)} =
\begin{cases} 
x_t^{(i)}, & \text{if } \alpha_i \geq u_i, \\
x_t^{(i)} - \Delta x_t^{(i)}, & \text{otherwise}.
\end{cases}
\]
This step ensures that the diffusion-driven expansion aligns with the posterior distribution, preserving the accuracy of the particle filter.

\textbf{Diffusion-Enhanced Particle Filtering}\\
By integrating exploratory particles, entropy-driven diffusion regularisation, and kernel-induced stochastic perturbations, the proposed framework creates a dynamic diffusion process that iteratively expands the effective support range. The recursive relationship for the support range becomes:
\[
\mathcal{S}_{t+1} = (\mathcal{S}_t \cup \mathcal{B}) \oplus h_{\text{opt}},
\]
where \( \oplus h_{\text{opt}} \) represents kernel-induced stochastic diffusion.

This diffusion framework overcomes the Prior Boundary Phenomenon by continuously extending the particle filter's exploration capability, enabling robust state estimation for target states \( x_g \notin \mathcal{S}_{\text{prior}} \).

\subsection{Proof of Support Range Expansion Beyond the Prior Boundary}

Given the enhancements of exploratory diffusion (\(\mathcal{B}\)), entropy-driven regularisation, and kernel-induced stochastic perturbations (\(\oplus h_{\text{opt}}\)), we formally prove that the proposed framework enables the support range \( \mathcal{S}_t \) to expand beyond the prior boundary \( \mathcal{S}_{\text{prior}} \), enabling the particle filter to explore regions where \( x_g \notin \mathcal{S}_{\text{prior}} \).

\begin{proposition}[Expansion of Support Range]
With the proposed enhancements, the support range \( \mathcal{S}_t \) satisfies the recursive relationship:
\[
\mathcal{S}_t^{\text{new}} = \mathcal{S}_t \cup \mathcal{B}, \quad \mathcal{S}_{t+1} = \mathcal{S}_t^{\text{new}} \oplus h_{\text{opt}},
\]
where \( \mathcal{B} \) is the extended bounding box sampled by exploratory particles, and \( \oplus h_{\text{opt}} \) represents the kernel-induced expansion. Starting from the initial prior boundary:
\[
\mathcal{S}_0 = \mathcal{S}_{\text{prior}},
\]
the recursive updates ensure that for any target state \( x_g \in \mathcal{B} \), there exists a time step \( t \) such that:
\[
x_g \in \mathcal{S}_t.
\]
\end{proposition}

\begin{proof}

The proof proceeds in three steps, corresponding to the three key mechanisms of the unified diffusion framework.

\textbf{Step 1: Base Case (\(t = 0\))}

At \(t = 0\), the particle support range is defined by the prior distribution:
\[
\mathcal{S}_0 = \mathcal{S}_{\text{prior}} = \{x_0 \in \mathcal{X} : P(x_0) > 0\}.
\]

This forms the initial support range, which is confined within the prior boundary. Thus, at \(t = 0\), we have:
\[
\mathcal{S}_0 = \mathcal{S}_{\text{prior}}.
\]

\textbf{Step 2: Exploratory Diffusion (\(\mathcal{B}\))}

At each subsequent time step \(t > 0\), a subset of particles (exploratory particles) is sampled from a uniform distribution over an extended bounding box \(\mathcal{B}\):
\[
x_t^{(j)} \sim \mathcal{U}(\mathcal{B}), \quad j \in \mathcal{E}.
\]

The exploratory particles extend the support range to include regions outside \(\mathcal{S}_t\), producing an intermediate expanded range:
\[
\mathcal{S}_t^{\text{new}} = \mathcal{S}_t \cup \mathcal{B}.
\]

Since \(\mathcal{B}\) is designed to include regions where \(x_g \notin \mathcal{S}_{\text{prior}}\), this step ensures:
\[
x_g \in \mathcal{S}_t^{\text{new}}, \quad \text{if } x_g \in \mathcal{B}.
\]

\textbf{Step 3: Kernel-Induced Stochastic Diffusion (\(\oplus h_{\text{opt}}\))}

Each particle in \(\mathcal{S}_t^{\text{new}}\) undergoes stochastic perturbation through a Gaussian kernel. For any particle \(x_t^{(i)}\), the perturbation is given by:
\[
\Delta x_t^{(i)} \sim h_{\text{opt}} \cdot \mathcal{L} \cdot \mathcal{N}(0, I),
\]
where:
\(h_{\text{opt}} = A \cdot N^{-\frac{1}{n+4}}\) is the bandwidth optimally adjusted to the particle count (\(N\)) and state space dimensionality (\(n\)),
\(\mathcal{L}\) is derived from the Cholesky decomposition of the covariance matrix \(\Sigma\), ensuring diffusion aligns with the local particle distribution.

The kernel-induced perturbation expands the support range locally around each particle, resulting in:
\[
\mathcal{S}_{t+1} = \mathcal{S}_t^{\text{new}} \oplus h_{\text{opt}}.
\]

This step ensures that even if \(x_g \notin \mathcal{B}\), the stochastic diffusion mechanism can incrementally expand the support range to include \(x_g\) over multiple iterations.

\textbf{Step 4: Inductive Conclusion}

By induction, assume that at time \(t\), the support range satisfies:
\[
x_g \in \mathcal{S}_t, \quad \text{if } x_g \in (\mathcal{S}_{t-1} \cup \mathcal{B}) \oplus h_{\text{opt}}.
\]

Using Steps 2 and 3:
1. Exploratory diffusion (\(\mathcal{B}\)) ensures that \(x_g \in \mathcal{S}_t^{\text{new}}\) if \(x_g \in \mathcal{B}\).
2. Kernel-induced diffusion (\(\oplus h_{\text{opt}}\)) ensures \(x_g \in \mathcal{S}_{t+1}\) by expanding the neighbourhood around \(\mathcal{S}_t^{\text{new}}\).

Combining these, we conclude:
\[
x_g \in \mathcal{S}_{t+1}, \quad \forall x_g \in \mathcal{B}.
\]

Starting from \(\mathcal{S}_0 = \mathcal{S}_{\text{prior}}\), the recursive updates \((\mathcal{S}_t \cup \mathcal{B}) \oplus h_{\text{opt}}\) ensure that for any \(x_g \in \mathcal{B}\), there exists a finite \(t\) such that:
\[
x_g \in \mathcal{S}_t.
\]

Thus, the proposed diffusion-driven framework systematically expands the support range \(\mathcal{S}_t\) through recursive integration of exploratory particles (\(\mathcal{B}\)) and kernel-induced stochastic perturbations (\(\oplus h_{\text{opt}}\)). This unified approach ensures that the particle filter is no longer confined to the initial prior boundary \(\mathcal{S}_{\text{prior}}\), enabling robust state estimation for target states \(x_g \notin \mathcal{S}_{\text{prior}}\).

\end{proof}
\begin{figure*}[htbp]
\centering
\subfigure[TPF: Particle Numbers and Exploration Ratios]{
  \includegraphics[width=0.43\linewidth,height=0.31\textheight]{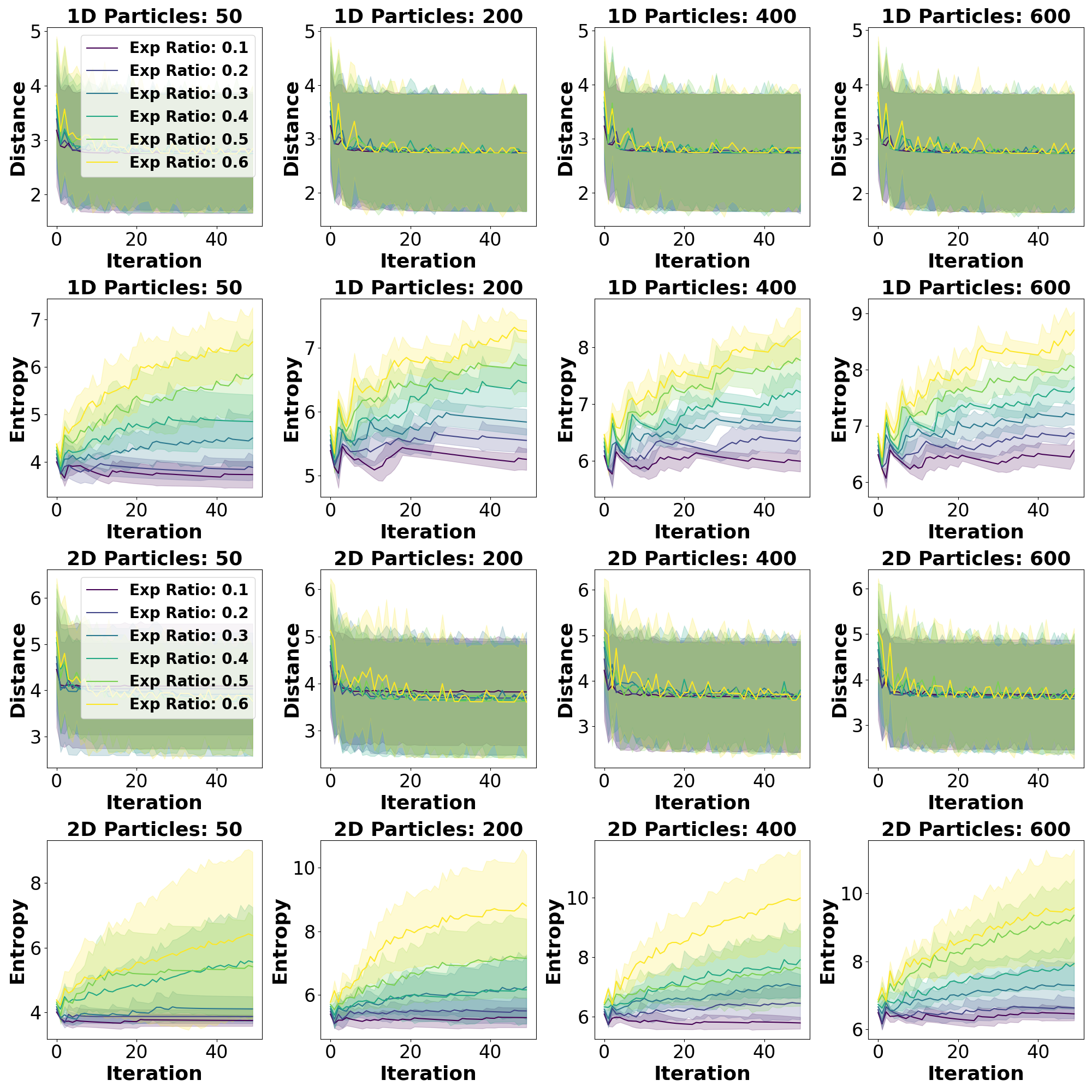}
\label{fig:Traditional_Particle_Filtering}}
\hfil
\subfigure[DEPF: Particle Numbers and Exploration Ratios]{\includegraphics[width=0.43\linewidth,height=0.31\textheight]{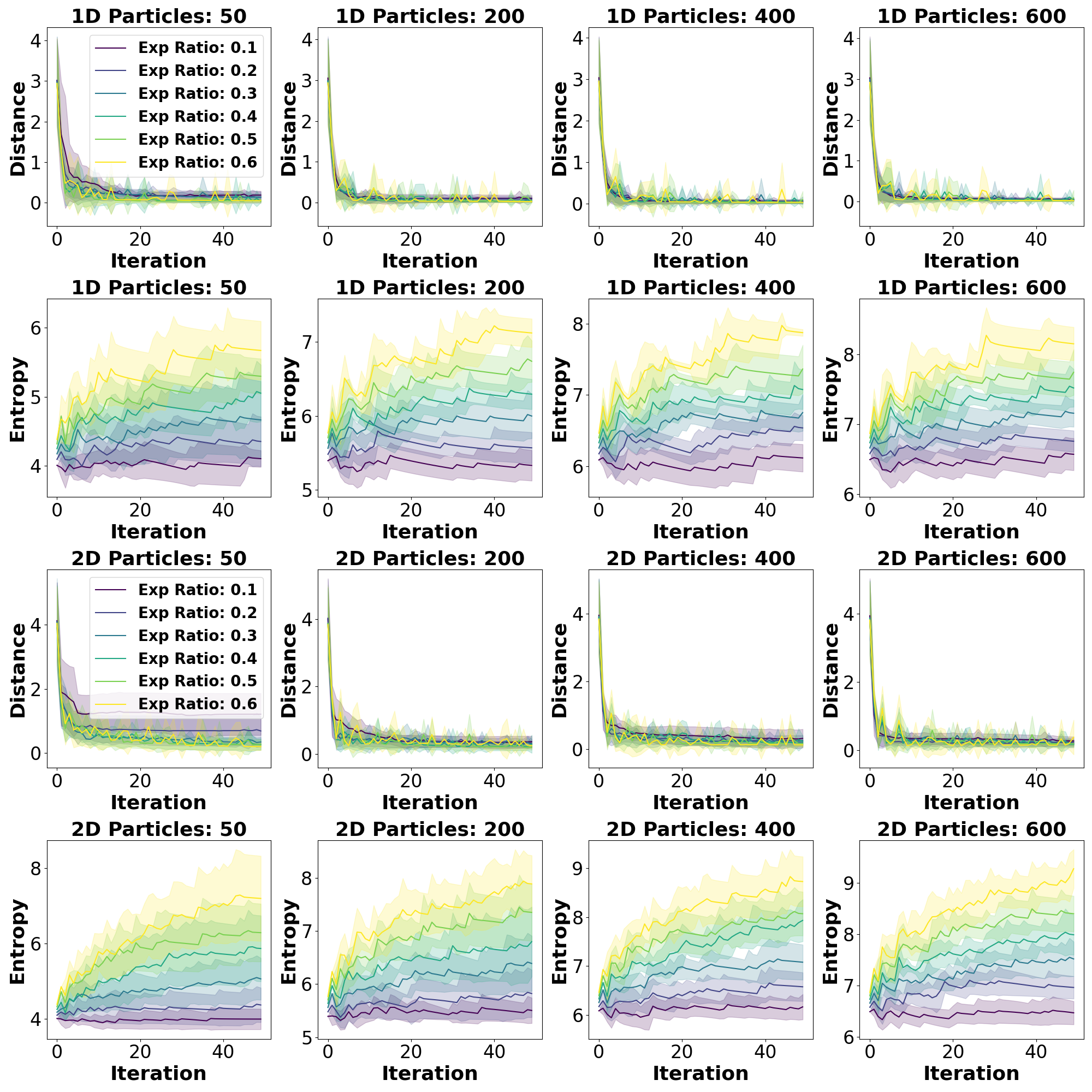}
  \label{fig:Diffusion-Enhanced Particle Filtering}}
\caption{Performance Results for 1D and 2D Scenarios}
\label{fig:PerformanceResults}
\end{figure*}
\section{Experiments}
This study consists of two experimental phases designed to evaluate the effectiveness of the proposed innovative  particle filtering method in addressing the \textit{Prior Boundary Phenomenon}.
\subsection{The Limitations of  Traditional Particle Filtering}

The first phase focuses on a simulated environment, where the initial particle distribution is confined within a restricted range, while the target state lies outside this boundary. The experiment systematically varies the number of particles, exploration ratios (ER), and dimensionality to assess the method's ability to expand the support range. \textbf{ER} is the proportion of particles added outside the prior to improve state estimation, adjustable for better exploration. It also compares the efficacy of Diffusion-Enhanced Particle Filtering (DEPF) with traditional particle filtering (TPF) techniques. \textbf{This phase primarily highlights the limitations of TPF methods in breaking through prior boundaries.}

Two primary metrics are used to evaluate the performance of the particle filters:
\textbf{Distance (D)} measures the Euclidean distance from the estimated state to the true goal. Lower \textbf{D}-values indicate more accurate state estimations, showing effective convergence. \textbf{Entropy (E)} calculates the entropy of the particle weights distribution. Higher \textbf{E}-values suggest a uniform weight distribution across particles, indicating effective exploration of the state space. Low \textbf{E}-values point to potential weight degeneracy, impacting filter performance. The \textit{Priori Scope} represents the ratio of the prior region to the entire state space, while the \textit{Exploration Ratio} denotes the proportion of the total particles specifically allocated for exploration. These concepts are unique to DEPF and do not apply to TPF, but are included here to facilitate direct comparison with DEPF in corresponding scenarios.
The results for 1-D and 2-D scenarios are presented in the main body of the text, as shown in Fig. \ref{fig:Traditional_Particle_Filtering} and Fig. \ref{fig:Diffusion-Enhanced Particle Filtering}. Results for scenarios extending from 3-D to 7-D are detailed in the Appendix~\ref{seb:Lab1}.

\textbf{Analysis of the Distance:}  
Across all 1-D to 7-D scenarios, DEPF consistently demonstrates superior performance over TPF in terms of the distance metric, achieving significantly lower final distance means. This indicates that DEPF provides more accurate state estimations, bringing the estimated positions closer to the true target. For instance, in the 1D scenario with 400 particles and an exploration ratio of 0.3, DEPF achieves a final distance mean of $0.07$, compared to $2.73$ for TPF. Similarly, in the 2-D scenario with 600 particles and an exploration ratio of 0.3, DEPF reduces the distance mean to $0.19$, while TPF remains at $3.58$. These results highlight DEPF's ability to overcome boundary constraints, particularly in cases where the target state lies beyond the initial particle distribution, effectively expanding the state-space exploration range.

\textbf{Analysis of the Entropy:}
In terms of the entropy metric, TPF generally exhibits higher entropy values compared to DEPF, primarily due to its confinement within boundary constraints. For example, in the 2-D scenario with 400 particles and an exploration ratio of 0.4, TPF records an entropy mean of $7.91$, while DEPF achieves $7.99$. The higher entropy observed in TPF reflects a uniform distribution of particle weights within the constrained boundaries (see Fig. \ref{fig:1Dand2D}), rather than meaningful exploration of the state space. In contrast, DEPF utilises exploration particles to extend beyond these boundaries, allowing particle weights to concentrate more effectively around regions of higher likelihood. This results in slightly lower entropy values, which indicate focused exploration and successful adaptation to dynamic environments. DEPF's ability to balance exploration and particle weight concentration underscores its efficacy in scenarios requiring robust boundary expansion.

\subsection{DEPF across Divergent Posterior Distributions}

The second phase involves a realistic scenario simulation based on a mobile robot tasked with locating an atmospheric release source. Particles are initialised using diverse prior distributions, including Uniform, Beta, Gaussian, Dirichlet, and four non-convex prior distributions, to represent various prior assumptions in real-world applications. 

The source is randomly positioned within the region of [10, 10] to [15, 15], while the agent starts from the region of [0, 0] to [5, 5]. We conducted 100 experiments and selected one instance to showcase the initial and final belief distributions of both TPF and DEPF. It was observed that TPF consistently fails to extend the belief beyond the prior boundary from Fig. \ref{fig:TPFIP} and \ref{fig:DEPFIP}.  

This study introduces two additional metrics: (1) the success rate (SR), and (2) the number of steps required for successful estimation. \textbf{The primary focus of this subsection is to demonstrate that DEPF can successfully estimate the target state across varying initial distributions, effectively overcoming the limitations imposed by the prior boundary.}

\begin{figure}[htp]
\centering
\subfigure[Initial prior of TPF]{
  \includegraphics[width=0.47\linewidth]{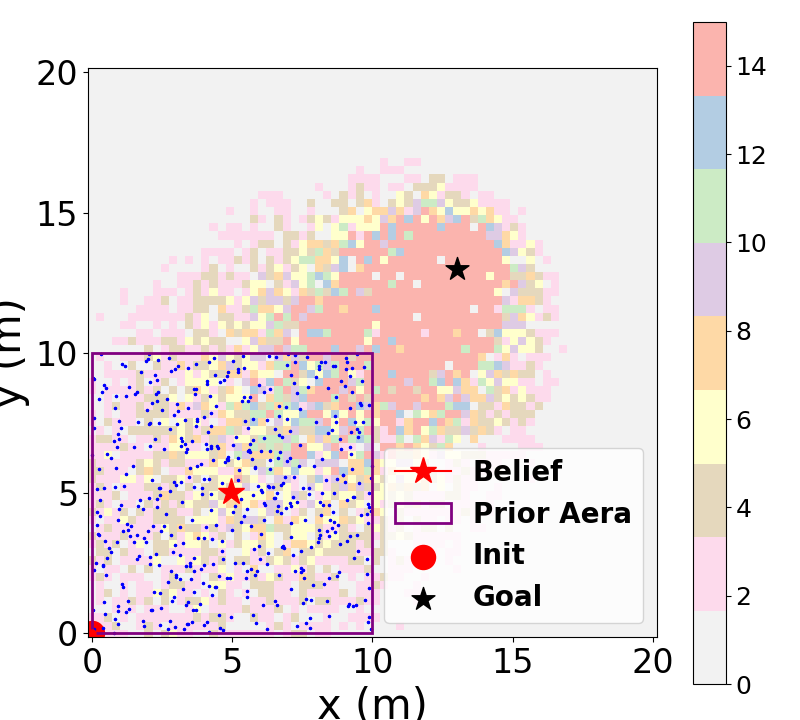}
  \label{fig:TPFIP}}
\hfil
\subfigure[Final distribution of TPF]{
  \includegraphics[width=0.47\linewidth]{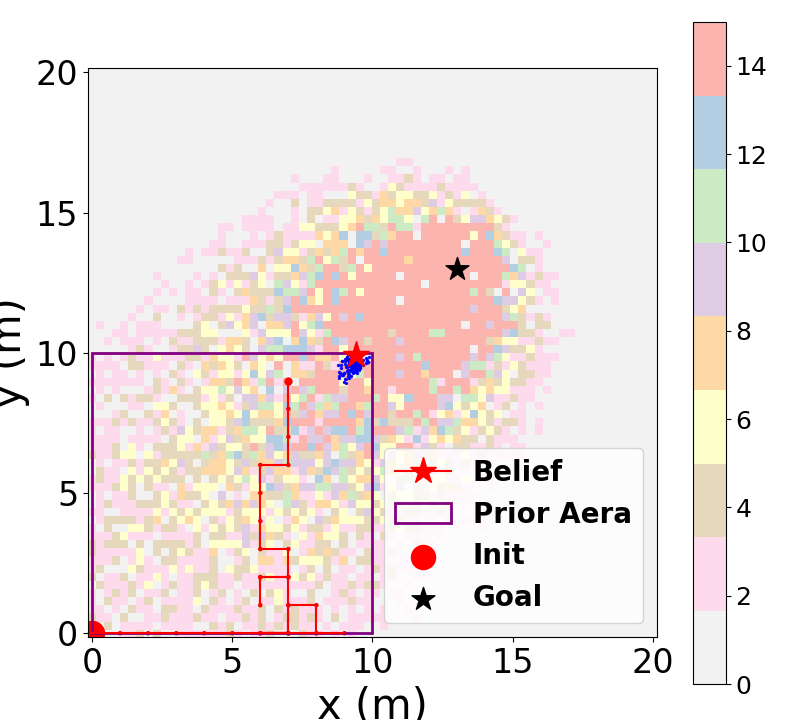}
  \label{fig:TPFFD}}
\hfil
\subfigure[Initial prior of DEPF]{
  \includegraphics[width=0.47\linewidth]{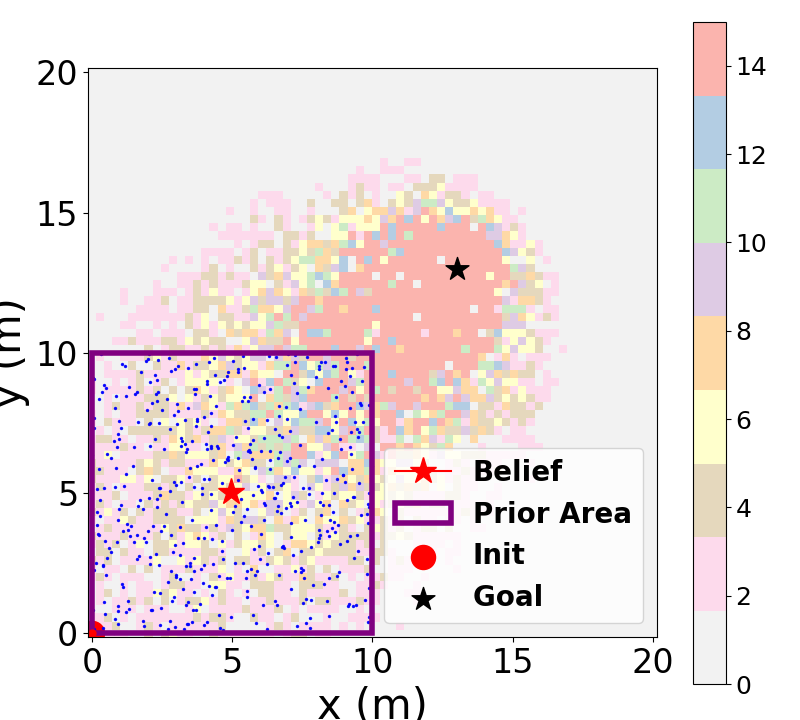}
  \label{fig:DEPFIP}}
\hfil
\subfigure[Final distribution of DEPF]{
  \includegraphics[width=0.47\linewidth]{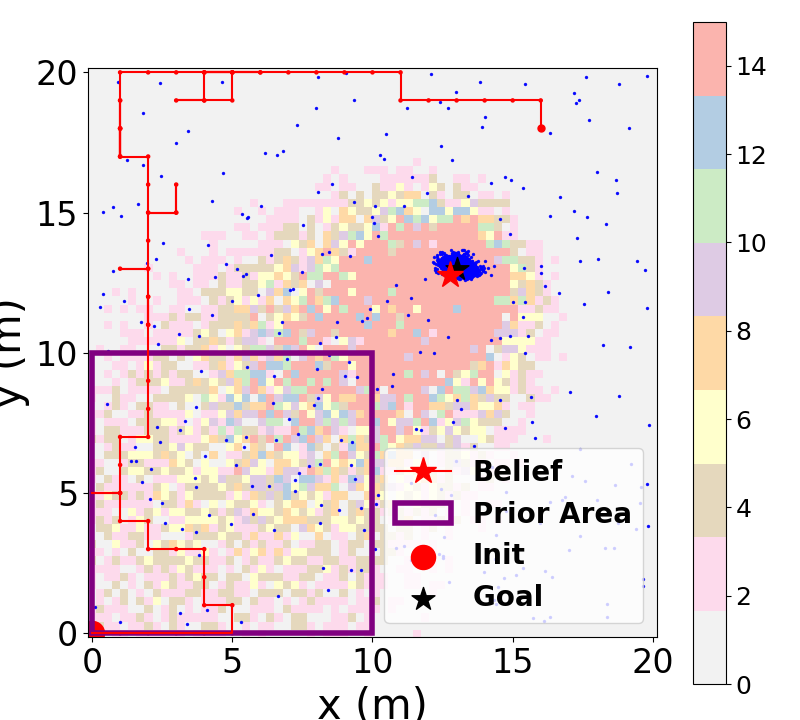}
  \label{fig:DEPFFD}}
\caption{TPF and DEPF in Scenario 2 }
\label{fig_PF_Visualisation}
\end{figure}

The experimental results presented in Table \ref{tab:0.3} focus on a specific case where the priori scope is $0.3$ and the exploration ratio is $0.3$. While the experiments were conducted across a range of ratios \([0.3, 0.4, 0.5, 0.6, 0.7, 0.8]\) and exploration ratios \([0.1, 0.2, 0.3, 0.4, 0.5, 0.6]\), only the results for this particular setting are provided here for clarity. The complete results for all combinations of ratios and exploration ratios are detailed in the Appendix \ref{seb:Lab2}. 

\begin{table}[h]
\centering
\caption{Experimental Results for Priori Scope = 0.3 at Ratio 0.3}
\label{tab:0.3}
{\footnotesize 
\begin{tabular}{|c|c|c|c|c|c|}
\hline
\textbf{Prior} & \textbf{SR} & \textbf{Entropy} & \textbf{Distance} & \textbf{Avg. Step} \\ \hline
Uniform & 0.81 & 6.79 $\pm$ 0.34 & 0.60 $\pm$ 0.62 & 67.29 \\ \hline
Beta & 0.81 & 6.87 $\pm$ 0.41 & 0.42 $\pm$ 0.32 & 67.36 \\ \hline
Gaussian & 0.87 & 7.17 $\pm$ 0.34 & 0.22 $\pm$ 0.05 & 71.38 \\ \hline
Dirichlet & 0.72 & 6.42 $\pm$ 0.30 & 0.37 $\pm$ 0.38 & 80.03 \\ \hline
Star & 0.84 & 6.82 $\pm$ 0.46 & 0.37 $\pm$ 0.15 & 70.05 \\ \hline
 ${1}/{4}$ Ring & 0.82 & 6.93 $\pm$ 0.35 & 0.39 $\pm$ 0.28 & 67.76 \\ \hline
${1}/{2}$ Ring & 0.83 & 7.13 $\pm$ 0.45 & 0.39 $\pm$ 0.21 & 69.17 \\ \hline
${3}/{4}$ Ring & 0.84 & 7.10 $\pm$ 0.31 & 0.44 $\pm$ 0.42 & 68.18 \\ \hline
\end{tabular}}
\end{table}

The results demonstrate that our DEPF method effectively mitigates boundary effects across a variety of prior distributions, including both common and non-convex priors, achieving high success rates and accurate state estimations even under challenging initial conditions. Among the common priors (Uniform, Beta, Gaussian, and Dirichlet), the Gaussian prior exhibits the best performance with the highest success rate (0.87) and lowest mean distance (0.22), indicating its ability to adapt effectively to boundary constraints. In contrast, the Dirichlet prior shows the weakest performance with the lowest success rate (0.72) and entropy (6.42), highlighting its limitations in exploring beyond the constrained boundaries.

For the non-convex priors (1/4 Ring, 1/2 Ring, 3/4 Ring, and Star), DEPF demonstrates its robustness by achieving consistent success across these complex prior shapes. The fractional ring priors, particularly the 1/2 Ring and 3/4 Ring, achieve high success rates (0.83 and 0.84, respectively) and competitive entropy values, showcasing DEPF's capability to effectively navigate irregular and asymmetric state spaces. These results emphasise DEPF's versatility and strength in addressing both common and non-convex priors, overcoming boundary limitations, and ensuring accurate and efficient state estimation across diverse initial distributions.

\section{Conclusion}
This study introduces the Prior Boundary Phenomenon as a critical limitation in particle filtering and proposes the Diffusion-Enhanced Particle Filtering (DEPF) Framework to address this challenge. By integrating exploratory particles, entropy-driven regularisation, and kernel-induced stochastic perturbations, DEPF systematically expands the particle support range and ensures robust estimation for target states beyond the prior's support. Comprehensive experiments validate DEPF's superior performance across diverse scenarios, highlighting its potential to extend particle filtering's applicability to complex and dynamic environments.

\clearpage
\bibliography{example_paper}
\bibliographystyle{icml2025}

\newpage

\section{Support Range}
Support Range is a term used in mathematics and statistics, typically to describe the range of non-zero or effective values of a function, distribution, or dataset. Specifically, its meaning depends on the context as follows:
\subsection{Support Range in Mathematical Analysis}
In mathematics, the \textit{support} of a function refers to the set of points where the function is non-zero.\\

Example:  
If a function \( f(x) \) is zero everywhere except on the interval \([a, b]\), its support range is \([a, b]\).

\[
\text{Support}(f) = \{ x \in \mathbb{R} : f(x) \neq 0 \}
\]

\subsection{Support Range in Statistics}
\textit{Support Range in Statistics}
In statistics, the support range describes the range or set of possible values a random variable can take.

\textit{a. Discrete Random Variable:}  
The support range is the set of all possible values.  
\textit{Example:} The result of rolling a die is:  
\[
\text{Support} = \{1, 2, 3, 4, 5, 6\}
\]

\textit{b. Continuous Random Variable:} 
The support range is the interval where the probability density function (PDF) is non-zero.  \\
\textit{Example:} For a standard normal distribution, the theoretical support is the entire real number line \(\mathbb{R}\), but the majority of the probability is concentrated within \([-3, 3]\).

\subsection{Support Range in Physics and Engineering}
In physics and signal processing, the support range describes the time or spatial region where a signal or waveform has non-zero values.\\  
\textit{Example:} A signal might exist only in the time range:  
\[
t \in [0, 10]
\]  
and is zero elsewhere.
\newpage
\section{Proof- Proposition: Recursive Confinement of Support Range}
\label{seb:ProofProposition}

\textbf{Proof and Statement of the Prior Boundary Phenomenon}

Given the definitions of the prior boundary \(\mathcal{S}_{\text{prior}}\) and the support range \(\mathcal{S}_t\), we formally prove that the recursive nature of particle filtering confines the support range within the prior boundary, thereby limiting its ability to explore regions outside \(\mathcal{S}_{\text{prior}}\).

\textbf{Proposition: Recursive Confinement of Support Range}

The support range \(\mathcal{S}_t\) at time \(t\) satisfies the recursive relationship:
\[
\mathcal{S}_t \subseteq \mathcal{S}_{t-1} \subseteq \cdots \subseteq \mathcal{S}_{\text{prior}}, \quad \forall t \geq 0,
\]
with the base case:
\[
\mathcal{S}_0 = \mathcal{S}_{\text{prior}}.
\]

Thus, for any time \(t \geq 0\), the support range is strictly confined by the prior boundary:
\[
\mathcal{S}_t \subseteq \mathcal{S}_{\text{prior}}.
\]

\textbf{Key Definitions}\\
1. \textit{Prior Boundary}: The initial support range is defined as:
   \[
   \mathcal{S}_{\text{prior}} = \{x_0 \in \mathcal{X} : P(x_0) > 0\},
   \]
   where \(P(x_0)\) is the prior distribution at \(t=0\), and \(\mathcal{X}\) denotes the entire state space.

2. \textit{Support Range}: At time \(t\), the support range \(\mathcal{S}_t\) is defined as the region of the state space where the particle filter has non-zero probability, i.e.,
   \[
   \mathcal{S}_t = \bigcup_{i=1}^N \{x_t \in \mathcal{X} : P(x_t | x_{t-1}^{(i)}) > 0\},
   \]
   where \(\{x_{t-1}^{(i)}\}_{i=1}^N\) are particles from \(\mathcal{S}_{t-1}\).

\textbf{Proof: Recursive Confinement of Support Range}
\begin{proof}
We now formally prove the proposition using induction, explicitly incorporating the three key steps of particle filtering: \textit{initialisation}, \textit{particle propagation}, and \textit{resampling and weight updates}.

\textbf{Step 1: Base Case at \(t = 0\)}

At \(t = 0\), particles are sampled from the prior distribution \(P(x_0)\). By definition:
\[
\mathcal{S}_0 = \mathcal{S}_{\text{prior}} = \{x_0 \in \mathcal{X} : P(x_0) > 0\}.
\]
This establishes the initial support range as \(\mathcal{S}_{\text{prior}}\), and we have:
\[
\mathcal{S}_0 \subseteq \mathcal{S}_{\text{prior}}.
\]

\textbf{Step 2: Inductive Hypothesis}

Assume that at time \(t-1\), the support range satisfies:
\[
\mathcal{S}_{t-1} \subseteq \mathcal{S}_{\text{prior}}.
\]

We will show that this implies:
\[
\mathcal{S}_t \subseteq \mathcal{S}_{\text{prior}}.
\]

\textbf{Step 3: Particle Propagation}

At time \(t\), particles \(\{x_t^{(i)}\}_{i=1}^N\) are generated by propagating particles from \(\mathcal{S}_{t-1}\) according to the state transition model \(P(x_t | x_{t-1})\). Formally:
\[
x_t^{(i)} \sim P(x_t | x_{t-1}^{(i)}), \quad x_{t-1}^{(i)} \in \mathcal{S}_{t-1}.
\]

Since \(x_{t-1}^{(i)} \in \mathcal{S}_{t-1}\), and by the inductive hypothesis \(\mathcal{S}_{t-1} \subseteq \mathcal{S}_{\text{prior}}\), we know:
\[
P(x_{t-1}^{(i)}) = 0 \quad \text{for } x_{t-1}^{(i)} \notin \mathcal{S}_{\text{prior}}.
\]

Furthermore, the state transition model \(P(x_t | x_{t-1})\) satisfies:
\[
P(x_t | x_{t-1}) = 0 \quad \text{if } x_{t-1} \notin \mathcal{S}_{t-1}.
\]

Thus, the propagated particles \(\{x_t^{(i)}\}\) must satisfy:
\[
x_t^{(i)} \in \mathcal{S}_{t-1}, \quad \forall i.
\]

Therefore:
\[
\mathcal{S}_t \subseteq \mathcal{S}_{t-1}.
\]

\textbf{Step 4: Weight Update and Resampling}

After particle propagation, the weights of the particles are updated based on the observation model \(P(y_t | x_t)\). Specifically:
\[
w_t^{(i)} \propto w_{t-1}^{(i)} P(y_t | x_t^{(i)}).
\]

Since \(P(y_t | x_t) = 0\) for \(x_t \notin \mathcal{S}_{t-1}\), the weights of particles outside \(\mathcal{S}_{t-1}\) are zero. During resampling, particles are resampled proportional to their weights, ensuring that:
\[
x_t^{(i)} \in \mathcal{S}_{t-1}, \quad \forall i.
\]

Thus, after resampling:
\[
\mathcal{S}_t \subseteq \mathcal{S}_{t-1}.
\]

\textbf{Step 5: Inductive Conclusion}

Combining the results of Steps 3 and 4, we have shown that:
\[
\mathcal{S}_t \subseteq \mathcal{S}_{t-1}.
\]

Using the inductive hypothesis \(\mathcal{S}_{t-1} \subseteq \mathcal{S}_{\text{prior}}\), we conclude:
\[
\mathcal{S}_t \subseteq \mathcal{S}_{\text{prior}}, \quad \forall t \geq 0.
\]

\textbf{Conclusion: Prior Boundary Phenomenon}

The proof demonstrates that the recursive nature of particle filtering enforces a strict confinement of the support range \(\mathcal{S}_t\) within the prior boundary \(\mathcal{S}_{\text{prior}}\). Specifically, if a target state \(x_g\) satisfies:
\[
x_g \notin \mathcal{S}_{\text{prior}},
\]
then \(x_g \notin \mathcal{S}_t\) for all \(t \geq 0\). 

This phenomenon, termed the \textit{Prior Boundary Phenomenon}, highlights a fundamental limitation of particle filtering: the prior distribution completely determines the accessible regions of the state space, and regions outside the prior boundary remain unreachable.

By definition, the support range $\mathcal{S}_t$ is given as:
\[
\mathcal{S}_t = \bigcup_{i=1}^N \{x_t \in \mathcal{X} : P(x_t | x_{t-1}^{(i)}) > 0\}, \quad \text{where } x_{t-1}^{(i)} \in \mathcal{S}_{t-1}.
\]
For the base case at $t = 0$, we have:
\[
\mathcal{S}_0 = \mathcal{S}_{\text{prior}} = \{x_0 \in \mathcal{X} : P(x_0) > 0\}.
\]
By induction, assume $\mathcal{S}_{t-1} \subseteq \mathcal{S}_{\text{prior}}$. Since particle propagation at time $t$ depends entirely on $x_{t-1}^{(i)} \in \mathcal{S}_{t-1}$ and the state transition model $P(x_t | x_{t-1})$, which cannot generate particles outside $\mathcal{S}_{t-1}$, it follows that:
\[
\mathcal{S}_t \subseteq \mathcal{S}_{t-1}.
\]
Combining this with the inductive hypothesis $\mathcal{S}_{t-1} \subseteq \mathcal{S}_{\text{prior}}$, we conclude:
\[
\mathcal{S}_t \subseteq \mathcal{S}_{\text{prior}}, \quad \forall t \geq 0.
\]
\end{proof}

\newpage
\section{Proof}
\subsection{Diffusion-Driven Support Range Expansion}

The Prior Boundary Phenomenon arises due to the recursive confinement of the particle support range \( \mathcal{S}_t \) within the initial prior boundary \( \mathcal{S}_{\text{prior}} \). If a target state \( x_g \notin \mathcal{S}_{\text{prior}} \), no particles can be generated near \( x_g \), resulting in the failure of state estimation. To address this, we propose a diffusion-driven particle filtering framework that incorporates dynamic exploration, entropy-driven diffusion regularisation, and kernel-based stochastic perturbations to expand the effective support range.

The first enhancement introduces exploratory particles to expand the effective support range beyond \( \mathcal{S}_{\text{prior}} \). At each time step, a subset of particles is sampled from a uniform distribution \( \mathcal{U}(\mathcal{B}) \) over an extended state space:
\[
x_t^{(j)} \sim \mathcal{U}(\mathcal{B}), \quad j \in \mathcal{E},
\]
where \( \mathcal{B} \) is a bounding box defining the extended state space, and \( \mathcal{E} \) represents the indices of exploratory particles. The weights of these exploratory particles are initialised as:
\[
w_t^{(j)} = \frac{\epsilon}{|\mathcal{E}|}, \quad \epsilon \ll 1.
\]
This ensures that while the primary particle set continues to approximate the posterior within the current support range, exploratory particles can sample regions beyond \( \mathcal{S}_{\text{prior}} \), increasing the likelihood of covering previously unreachable states.

To preserve diversity in the particle set, entropy regularisation is applied during the weight update step. This prevents particles with smaller weights from being prematurely eliminated, ensuring that even low-probability regions of the state space are adequately explored. The weights are updated with an entropy term as:
\[
w_t^{(i)} = w_t^{(i)} + \beta H,
\]
where \( H \) is the entropy of the weight distribution, defined as:
\[
H = -\sum_{i=1}^N w_t^{(i)} \log(w_t^{(i)} + \epsilon).
\]
By penalising overly concentrated weight distributions, this mechanism maintains a more balanced representation of the posterior, mitigating the degeneracy problem common in particle filters.

Kernel-based perturbations are then introduced to further expand the particle support range dynamically. Each particle \( x_t^{(i)} \) is perturbed using a Gaussian kernel to simulate stochastic exploration of the local neighbourhood. The perturbation is generated as:
\[
\Delta x_t^{(i)} \sim h_{\text{opt}} \cdot \mathcal{L} \cdot \mathcal{N}(0, I),
\]
where \( \mathcal{L} \) is the lower triangular matrix obtained from the Cholesky decomposition of the covariance matrix \( \Sigma \), such that: \( \Sigma = \mathcal{L} \mathcal{L}^T.\) And \( h_{\text{opt}} \) is the kernel bandwidth, dynamically adjusted based on the particle count \( N \) and the dimensionality \( n \) of the state space. The optimal bandwidth is calculated as:
\[
h_{\text{opt}} = A \cdot N^{-\frac{1}{n+4}}, \quad A = \left(\frac{4}{n+2}\right)^{\frac{1}{n+4}}.
\]

This ensures that the perturbation scale adapts to the number of particles and the complexity of the state space. Larger \( N \) leads to finer perturbations, allowing for precise adjustments, while smaller \( N \) increases the perturbation scale to improve exploration.

The covariance matrix \( \Sigma \) required for perturbation is computed dynamically from the particle set:
\[
\Sigma = \sum_{i=1}^N w_t^{(i)} (x_t^{(i)} - \mu)(x_t^{(i)} - \mu)^T + \lambda I,
\]
where \( \mu \) is the weighted mean of the particles:
\[
\mu = \sum_{i=1}^N w_t^{(i)} x_t^{(i)},
\]
and \( \lambda > 0 \) is a regularisation term that ensures positive definiteness of \( \Sigma \). Each particle is then updated as:
\[
x_t^{(i)} \leftarrow x_t^{(i)} + \Delta x_t^{(i)}.
\]

Finally, an MCMC acceptance criterion ensures that the perturbed particles remain consistent with the target posterior distribution. The acceptance probability for each perturbed particle \( x_t^{(i)} \) is calculated as:
\[
\alpha_i = \frac{w_{\text{new}}^{(i)}}{w_{\text{old}}^{(i)}} \cdot \exp\left(-\frac{1}{2} \Delta x_t^{(i)^T} \Sigma^{-1} \Delta x_t^{(i)}\right).
\]
A random uniform variable \( u_i \sim \mathcal{U}(0, 1) \) determines whether the particle is accepted:
\[
x_t^{(i)} =
\begin{cases} 
x_t^{(i)}, & \text{if } \alpha_i \geq u_i, \\
x_t^{(i)} - \Delta x_t^{(i)}, & \text{otherwise}.
\end{cases}
\]

By integrating exploratory particles, entropy regularisation, and kernel-based perturbations, the proposed framework overcomes the Prior Boundary Phenomenon. Exploratory particles extend the effective support range, entropy regularisation ensures weight diversity, and kernel perturbations enable dynamic adaptation to local structures. Together, these enhancements ensure that target states \( x_g \) outside \( \mathcal{S}_{\text{prior}} \) can be reached, allowing for robust state estimation in previously inaccessible regions.

\newpage

\newpage

\subsection{Proof of Support Range Expansion Beyond the Prior Boundary}

Given the enhancements of adaptive exploration, entropy regularisation, and kernel-based perturbations, we prove that the recursive nature of the proposed framework allows the support range \( \mathcal{S}_t \) to expand beyond the prior boundary \( \mathcal{S}_{\text{prior}} \), enabling the particle filter to explore regions where \( x_g \notin \mathcal{S}_{\text{prior}} \).

\begin{proposition}[Expansion of Support Range]
With the proposed enhancements, the support range \( \mathcal{S}_t \) satisfies the recursive relationship:
\[
\mathcal{S}_t^{\text{new}} = \mathcal{S}_t \cup \mathcal{B}, \quad \mathcal{S}_{t+1} = \mathcal{S}_t^{\text{new}} \oplus h_{\text{opt}},
\]
where \( \mathcal{B} \) is the extended bounding box sampled by exploratory particles, and \( \oplus h_{\text{opt}} \) represents the kernel-induced expansion. Starting from the initial prior boundary:
\[
\mathcal{S}_0 = \mathcal{S}_{\text{prior}},
\]
the recursive updates ensure that for any target state \( x_g \in \mathcal{B} \), there exists a time step \( t \) such that:
\[
x_g \in \mathcal{S}_t.
\]
\end{proposition}

\begin{proof}
The proposed enhancements modify the recursive particle filtering process in three key ways:

\textit{1.Exploratory Particles:}
At each time step \( t \), a subset of exploratory particles is sampled from a uniform distribution \( \mathcal{U}(\mathcal{B}) \), where \( \mathcal{B} \) represents an extended bounding box of the state space. These exploratory particles contribute to an expanded support range:
   \[
   \mathcal{S}_t^{\text{new}} = \mathcal{S}_t \cup \mathcal{B}.
   \]
   Since \( \mathcal{B} \) is designed to include regions outside \( \mathcal{S}_{\text{prior}} \), the effective support range \( \mathcal{S}_t^{\text{new}} \) can cover states \( x_g \notin \mathcal{S}_{\text{prior}} \).

\textit{2.Entropy Regularisation:}
The addition of an entropy term \( H \) during the weight update step ensures that particles with smaller weights are not prematurely eliminated:
   \[
   w_t^{(i)} = w_t^{(i)} + \beta H, \quad H = -\sum_{i=1}^N w_t^{(i)} \log(w_t^{(i)} + \epsilon).
   \]
This mechanism maintains particle diversity and prevents over-concentration in regions confined by \( \mathcal{S}_{\text{prior}} \). Consequently, exploratory particles retain sufficient influence to contribute to the posterior distribution, ensuring \( \mathcal{S}_t^{\text{new}} \) remains expanded.

\textit{3.Kernel-Based Perturbations:}
Each particle \( x_t^{(i)} \) is perturbed using a Gaussian kernel, which enables stochastic exploration of local neighbourhoods. The perturbation is generated as:
   \[
   \Delta x_t^{(i)} \sim h_{\text{opt}} \cdot \mathcal{L} \cdot \mathcal{N}(0, I),
   \]
   where \( h_{\text{opt}} = A \cdot N^{-\frac{1}{n+4}} \) dynamically adjusts the perturbation scale, and \( \mathcal{L} \) is derived from the Cholesky decomposition of the covariance matrix \( \Sigma \) (\( \Sigma = \mathcal{L} \mathcal{L}^T \)). This introduces an additional expansion to the support range:
   \[
   \mathcal{S}_{t+1} = \mathcal{S}_t^{\text{new}} \oplus h_{\text{opt}},
   \]
   where \( \oplus h_{\text{opt}} \) represents the kernel-induced expansion. Over multiple iterations, this perturbation ensures that \( \mathcal{S}_t \) can stochastically expand to include \( x_g \).

By combining these mechanisms, the recursive relationship for the support range becomes:
\[
\mathcal{S}_{t+1} = (\mathcal{S}_t \cup \mathcal{B}) \oplus h_{\text{opt}}.
\]
Starting from the initial prior boundary \( \mathcal{S}_0 = \mathcal{S}_{\text{prior}} \), it follows that for any target state \( x_g \in \mathcal{B} \), there exists a time step \( t \) such that:
\[
x_g \in \mathcal{S}_t.
\]
Thus, the support range is no longer confined by \( \mathcal{S}_{\text{prior}} \), and the particle filter can estimate states outside the prior boundary.

\end{proof}

\newpage
\section{Extended Related Work}


\subsection{Particle Filtering}
Particle filtering (PF) \cite{isard1996contour} is initially designed for object tracking in cluttered environments, and has shown broad application in diverse domains like hydrology \cite{moradkhani2005uncertainty}, mobile robot \cite{fox2001particle,shi2024autonomous,shi2024reinforcement}, geophysics \cite{van2009particle}, etc.. Technically, particle filtering uses Bayesian inference to estimate the dynamic states within the system \cite{soto2005self}: First initializes a series of particles randomly based on prior knowledge. Then iteratively refines these particles through prediction, update, normalization, and resampling steps. The final goal is to expect these particles to converge to approximate the system's posterior distribution, which can be used for system state estimation.

The efficiency and accuracy of the particle filter are generally discussed and optimized in iteration steps, such as comparing different resampling methods \cite{douc2005comparison, li2015resampling}, making the number of particles sampled adaptive \cite{fox2001kld, soto2005self}, among others. However, these optimizations assume that the target state aligns with the prior's support. Cases where the target state lies beyond the prior boundaries are an underexplored area in PF.

\subsection{Bayesian Inference}
The broader Bayesian inference literature provides more insights into prior distribution discussions.
\subsubsection{Prior Selections in Bayesian Inference}

A prior is a probability distribution of a parameter before observing data. 
Informative priors can be difficult to obtain, and an inappropriate prior can lead to misleading Bayesian inferences, thus making prior selection important and challenging \cite{berger1990robust, richardson1997some}. Discussions in prior selection include choosing an appropriate distribution and defining its support range.

For prior type selections, \cite{winkler1967assessment} examines methods for assessing priors in Bayesian analysis, aiming to quantify prior knowledge into a probability distribution.
\cite{gelman2017prior} proposes an improvement of integrating prior selection into the entire Bayesian workflow.

For challenges arising from prior boundary constraints, an intuitive solution is to apply unbounded distributions as priors, allowing the shape to maximize coverage of the target range. While the normal distribution is a common unbounded distribution option, it performs unideal with non-normal data. The Johnson unbounded distribution \cite{johnson1995continuous} is proposed to address this limitation by enabling transformations that handle skewness and kurtosis to fit various distribution shapes. \cite{sadok2023non, hartigan1996locally} explores the use of unbounded uniform priors, which assigns equal weight to all possible values. 
The application of Bayesian inference with these distributions includes scenarios with small samples \cite{marhadi2012quantifying}, automatic threshold setting for distribution-uncertain wind turbine monitoring systems \cite{marhadi2014using}, etc..
However, using unbounded distributions as priors requires significantly more computing resources and time to calculate posterior probabilities \cite{doshi2009large}, and makes their application challenging in large-scale data scenarios.

\subsubsection{Out of Distribution in Bayesian Inference}

The regression and classification tasks in machine learning also discuss prior boundary constraints. Data points that exceed prior knowledge distributions are labeled as anomalies or out-of-distribution (OOD) samples. Models' generalization and detection ability are investigated in these OOD samples. 
NatPN \cite{charpentier2021natural} parameterizes conjugate prior distributions to enable efficient uncertainty estimation. Posterior network \cite{charpentier2020posterior} applies Normalizing Flows to estimate a precise posterior distribution for samples. \cite{wang2021bayesian} incorporated aleatoric uncertainty and outlier exposure into Bayesian OOD detection.

\newpage
\section{Phase 1 Experimental Results Figure and Table}
\label{seb:Lab1}

\subsection{Experiment Description: Monte Carlo Particle Filtering for Source Term Estimation} 

\textbf{Objective}  
The experiment aims to evaluate the effectiveness of an \textit{enhanced Monte Carlo Markov Chain Particle Filter (MCMC-PF)} for \textit{source term estimation} in different-dimensional search spaces. The primary goal is to assess how \textit{particle count} and \textit{exploration ratio} impact estimation accuracy and convergence in a probabilistic search setting.

\textbf{Experimental Setup}   \\
\begin{itemize}
    \item \textit{Scenarios} 
    \begin{itemize}
        \item The experiment is conducted in \textbf{seven different dimensions:}\\\textbf{1D, 2D, 3D, 4D, 5D, 6D, and 7D} search spaces.
        \item Each scenario defines a \textbf{bounded search region} where the source location must be estimated.
        \item A \textbf{target source location (goal)} is randomly generated within a pre-defined range in each scenario.
        
    \end{itemize}

    \item \textit{Particles and Exploration Strategy}
    \begin{itemize}
        \item The experiment uses a \textit{Monte Carlo Particle Filtering approach}, enhanced with:
        \\
        - \textit{Systematic Resampling}: Ensures particle diversity and reduces degeneracy.  \\
        - \textit{Exploration Particles}: A fraction of particles is randomly regenerated to encourage global exploration.  \\
        - \textit{ntropy Regularization}: Adjusts particle weights to maintain a balanced distribution. 
        \item The number of particles tested: 50, 200, 300, 400, 600, 700, 800, 900, 1000  
        \item The exploration ratios tested: 0.1, 0.2, 0.3, 0.4, 0.5, 0.6**
    \end{itemize}
\end{itemize}

\textbf{Algorithm Workflow}
\begin{enumerate}
\item \textbf{Initialization}
\begin{itemize}
\item \textbf{Particles} are initialized randomly in the search space.
\item \textbf{Weights} are assigned equally to all particles.
\item A \textbf{target goal} (true source location) is selected randomly within the predefined range.
\end{itemize}

\item \textbf{Particle Update (Iteration Loop for 50 Iterations per Run)}
\begin{itemize}
    \item \textbf{Likelihood Computation}: Each particle's probability is updated based on its distance from the target source location.
    \item \textbf{Weight Adjustment}: Weights are recomputed based on a p-norm distance function.
    \item \textbf{Resampling}: If the \textbf{Effective Sample Size (ESS)} falls below a threshold, systematic resampling is performed.
    \item \textbf{Exploration Boost}: New particles are introduced into the search space to prevent stagnation.
    \item \textbf{Entropy Regularization}: Particle weights are slightly perturbed to prevent over-concentration.
\end{itemize}

\item \textbf{Performance Metrics Collected}
\begin{itemize}
    \item \textbf{Distance to Goal}: Measures how accurately the estimated source converges to the true goal.
    \item \textbf{Entropy of Weights}: Tracks how well the particle distribution maintains diversity.
\end{itemize}

\end{enumerate}

\textbf{Results Collection \& Visualization}
\begin{enumerate}
\item \textbf{Performance Evaluation}
\begin{itemize}
\item For each scenario and each combination of \textbf{particle count} and \textbf{exploration ratio}, the algorithm runs \textbf{10 trials}.
\item The final results include:
\begin{itemize}
\item \textbf{Average final distance to the goal}
\item \textbf{Standard deviation of final distance}
\item \textbf{Average final entropy}
\item \textbf{Standard deviation of entropy}
\end{itemize}
\end{itemize}

\item \textbf{Visualization}
\begin{itemize}
    \item \textbf{Line plots} depict the convergence behavior over iterations for different parameter configurations.
    \item \textbf{Shaded areas} represent the variance in results across trials.
    \item \textbf{Tables summarize} the numerical results for different configurations.
\end{itemize}

\end{enumerate}

This experiment explores how different \textit{particle filter configurations} affect the performance of \textit{source term estimation} across varying dimensions. By incorporating \textit{adaptive resampling, exploration strategies, and entropy regularization}, the algorithm demonstrates \textit{robust estimation accuracy and adaptability} in complex, high-dimensional search spaces.

\begin{table}[htbp]
\centering
\caption{
Experimental Results of Diffusion-Enhanced Particle Filtering for 1D and 2D Scenarios: Particle Numbers and Exploration Ratios
}
\label{tab:1D_2D_results_DEPF}
{\scriptsize  
\begin{tabular}{@{}ccccccc@{}}
\toprule
Scenario & Num Particles & Exploration Ratio & Final Distance Mean & Final Distance Std & Final Entropy Mean & Final Entropy Std \\ \midrule
\multirow{36}{*}{1D} & \multirow{6}{*}{50} & 0.1 & 0.1872 & 0.1019 & 4.104 & 0.1203 \\
                                &                     & 0.2 & 0.1346 & 0.1368 & 4.352 & 0.3593 \\
                                &                     & 0.3 & 0.1393 & 0.0827 & 4.664 & 0.5593 \\
                                &                     & 0.4 & 0.0381 & 0.0334 & 5.051 & 0.4104 \\
                                &                     & 0.5 & 0.0530 & 0.0373 & 5.296 & 0.2541 \\
                                &                     & 0.6 & 0.0501 & 0.0463 & 5.669 & 0.4241 \\
                                & \multirow{6}{*}{200} & 0.1 & 0.1037 & 0.0823 & 5.330 & 0.2072 \\
\bottomrule
\end{tabular}
}
\end{table}

\begin{table}[htbp]
\centering
\caption{Experimental Results of Traditional Particle Filtering for 1D and 2D Scenarios: Particle Numbers and Exploration Ratios}
\label{tab:1D_2D_results}
{\scriptsize  
\begin{tabular}{@{}ccccccc@{}}
\toprule
Scenario & Num Particles & Exploration Ratio & Final Distance Mean & Final Distance Std & Final Entropy Mean & Final Entropy Std \\ \midrule
\multirow{36}{*}{1D} & \multirow{6}{*}{50} & 0.1 & 2.7540 & 1.0959 & 3.7385 & 0.2789 \\
                      &                     & 0.2 & 2.7952 & 1.0854 & 3.8902 & 0.2644 \\
                      &                     & 0.3 & 2.8372 & 1.0549 & 4.5076 & 0.5816 \\
                      &                     & 0.4 & 2.7568 & 1.0864 & 4.8545 & 0.5691 \\
                      &                     & 0.5 & 2.8312 & 1.1617 & 5.8476 & 0.9593 \\
                      &                     & 0.6 & 2.8431 & 1.1640 & 6.5307 & 0.7251 \\
                      & \multirow{6}{*}{200} & 0.1 & 2.7448 & 1.0975 & 5.2568 & 0.1678 \\
\bottomrule
\end{tabular}
}
\end{table}


\begin{table}[htbp]
\centering
\caption{Experimental Results for 1D Scenarios: Particle Numbers and Exploration Ratios}
\label{tab:1D_results}
{\scriptsize  
\begin{tabular}{@{}ccccccc@{}}
\toprule
Scenario & Num Particles & Exploration Ratio & Final Distance Mean & Final Distance Std & Final Entropy Mean & Final Entropy Std \\ \midrule
\multirow{36}{*}{1D} & \multirow{6}{*}{50} & 0.1 & 0.3014 & 0.2687 & 3.964 & 0.2547 \\
                      &                     & 0.2 & 0.2069 & 0.1868 & 4.308 & 0.2653 \\
                      &                     & 0.3 & 0.0901 & 0.0654 & 4.704 & 0.1912 \\
                      &                     & 0.4 & 0.0815 & 0.0476 & 5.092 & 0.2846 \\
                      &                     & 0.5 & 0.0845 & 0.0707 & 5.277 & 0.3008 \\
                      &                     & 0.6 & 0.0940 & 0.0611 & 5.696 & 0.3891 \\
                      & \multirow{6}{*}{200} & 0.1 & 0.1161 & 0.1103 & 5.384 & 0.2369 \\
\bottomrule
\end{tabular}
}
\end{table}

\begin{table}[htbp]
\centering
\caption{Experimental Results for 2D Scenarios: Particle Numbers and Exploration Ratios}
\label{tab:2D_results}
{\scriptsize  
\begin{tabular}{@{}ccccccc@{}}
\toprule
Scenario & Num Particles & Exploration Ratio & Final Distance Mean & Final Distance Std & Final Entropy Mean & Final Entropy Std \\ \midrule
\multirow{36}{*}{2D} & \multirow{6}{*}{50} & 0.1 & 0.844962 & 0.426268 & 4.041065 & 0.375523 \\
                      &                     & 0.2 & 0.526593 & 0.271830 & 4.516154 & 0.611390 \\
                      &                     & 0.3 & 0.512732 & 0.565652 & 5.017112 & 0.833175 \\
                      &                     & 0.4 & 0.560432 & 0.476526 & 5.606400 & 0.651340 \\
                      &                     & 0.5 & 0.312028 & 0.211752 & 5.586065 & 1.231073 \\
                      &                     & 0.6 & 0.261207 & 0.131344 & 6.947933 & 0.771016 \\
                      & \multirow{6}{*}{200} & 0.1 & 0.442295 & 0.179295 & 5.546453 & 0.178802 \\
\bottomrule
\end{tabular}
}
\end{table}

\begin{table}[htbp]
\centering
\caption{Experimental Results for 3D Scenarios: Particle Numbers and Exploration Ratios}
\label{tab:3D_results}
{\scriptsize  
\begin{tabular}{@{}ccccccc@{}}
\toprule
Scenario & Num Particles & Exploration Ratio & Final Distance Mean & Final Distance Std & Final Entropy Mean & Final Entropy Std \\ \midrule
\multirow{36}{*}{3D} & \multirow{6}{*}{50} & 0.1 & 2.798 & 1.361 & 3.972 & 0.172 \\
                      &                     & 0.2 & 1.429 & 0.429 & 4.489 & 0.490 \\
                      &                     & 0.3 & 1.449 & 0.700 & 4.857 & 1.077 \\
                      &                     & 0.4 & 1.041 & 0.676 & 5.614 & 1.333 \\
                      &                     & 0.5 & 0.842 & 0.491 & 6.530 & 1.811 \\
                      &                     & 0.6 & 0.567 & 0.277 & 6.952 & 2.373 \\
                      & \multirow{6}{*}{200} & 0.1 & 1.354 & 0.634 & 5.309 & 0.225 \\
\bottomrule
\end{tabular}
}
\end{table}

\begin{table}[htbp]
\centering
\caption{Experimental Results for 4D Scenarios: Particle Numbers and Exploration Ratios}
\label{tab:4D_results}
{\scriptsize  
\begin{tabular}{@{}ccccccc@{}}
\toprule
Scenario & Num Particles & Exploration Ratio & Final Distance Mean & Final Distance Std & Final Entropy Mean & Final Entropy Std \\ \midrule
\multirow{60}{*}{4D} & \multirow{6}{*}{50} & 0.1 & 3.942 & 1.868 & 3.907 & 0.258 \\
                      &                     & 0.2 & 2.829 & 0.726 & 4.205 & 0.391 \\
                      &                     & 0.3 & 2.537 & 1.028 & 4.711 & 0.837 \\
                      &                     & 0.4 & 2.016 & 0.798 & 5.459 & 1.264 \\
                      &                     & 0.5 & 1.338 & 0.857 & 6.523 & 2.371 \\
                      &                     & 0.6 & 1.090 & 0.736 & 7.686 & 2.381 \\
                      & \multirow{6}{*}{200} & 0.1 & 2.377 & 1.000 & 5.432 & 0.218 \\
\bottomrule
\end{tabular}
}
\end{table}

\begin{table}[htbp]
\centering
\caption{Experimental Results for 5D Scenarios: Particle Numbers and Exploration Ratios}
\label{tab:5D_results}
{\scriptsize  
\begin{tabular}{@{}ccccccc@{}}
\toprule
Scenario & Num Particles & Exploration Ratio & Final Distance Mean & Final Distance Std & Final Entropy Mean & Final Entropy Std \\ \midrule
\multirow{60}{*}{5D} & \multirow{6}{*}{50} & 0.1 & 4.565 & 1.387 & 3.992 & 0.137 \\
                      &                     & 0.2 & 3.703 & 1.103 & 4.235 & 0.257 \\
                      &                     & 0.3 & 2.839 & 0.811 & 4.732 & 0.833 \\
                      &                     & 0.4 & 2.485 & 0.518 & 4.938 & 0.684 \\
                      &                     & 0.5 & 2.138 & 1.119 & 6.422 & 2.595 \\
                      &                     & 0.6 & 2.229 & 0.703 & 5.902 & 2.178 \\
                      & \multirow{6}{*}{200} & 0.1 & 3.662 & 1.032 & 5.431 & 0.144 \\
\bottomrule
\end{tabular}
}
\end{table}

\begin{table}[htbp]
\centering
\caption{Experimental Results for 6D Scenarios: Particle Numbers and Exploration Ratios}
\label{tab:6D_results}
{\scriptsize  
\begin{tabular}{@{}ccccccc@{}}
\toprule
Scenario & Num Particles & Exploration Ratio & Final Distance Mean & Final Distance Std & Final Entropy Mean & Final Entropy Std \\ \midrule
\multirow{60}{*}{6D} & \multirow{6}{*}{50} & 0.1 & 5.796 & 1.619 & 4.019 & 0.144 \\
                      &                     & 0.2 & 4.546 & 1.229 & 4.224 & 0.202 \\
                      &                     & 0.3 & 4.105 & 1.197 & 4.564 & 0.642 \\
                      &                     & 0.4 & 3.424 & 0.782 & 5.548 & 1.687 \\
                      &                     & 0.5 & 2.864 & 1.131 & 6.256 & 2.197 \\
                      &                     & 0.6 & 3.450 & 0.985 & 4.947 & 0.610 \\
                      & \multirow{6}{*}{200} & 0.1 & 3.779 & 0.871 & 5.491 & 0.119 \\
\bottomrule
\end{tabular}
}
\end{table}

\begin{table}[htbp]
\centering
\caption{Experimental Results for 7D Scenarios: Particle Numbers and Exploration Ratios}
\label{tab:7D_results}
{\scriptsize  
\begin{tabular}{@{}ccccccc@{}}
\toprule
Scenario & Num Particles & Exploration Ratio & Final Distance Mean & Final Distance Std & Final Entropy Mean & Final Entropy Std \\ \midrule
\multirow{60}{*}{7D} & \multirow{6}{*}{50} & 0.1 & 5.417 & 1.294 & 4.015 & 0.148 \\
                      &                     & 0.2 & 5.901 & 1.717 & 4.271 & 0.301 \\
                      &                     & 0.3 & 4.534 & 1.233 & 4.442 & 0.240 \\
                      &                     & 0.4 & 4.810 & 1.106 & 4.326 & 0.306 \\
                      &                     & 0.5 & 4.307 & 1.254 & 4.817 & 0.572 \\
                      &                     & 0.6 & 3.873 & 0.492 & 5.195 & 0.520 \\
                      & \multirow{6}{*}{200} & 0.1 & 4.537 & 0.939 & 5.428 & 0.230 \\
\bottomrule
\end{tabular}
}
\end{table}

\begin{figure}[htbp]
\centering
\includegraphics[width=0.8\linewidth]{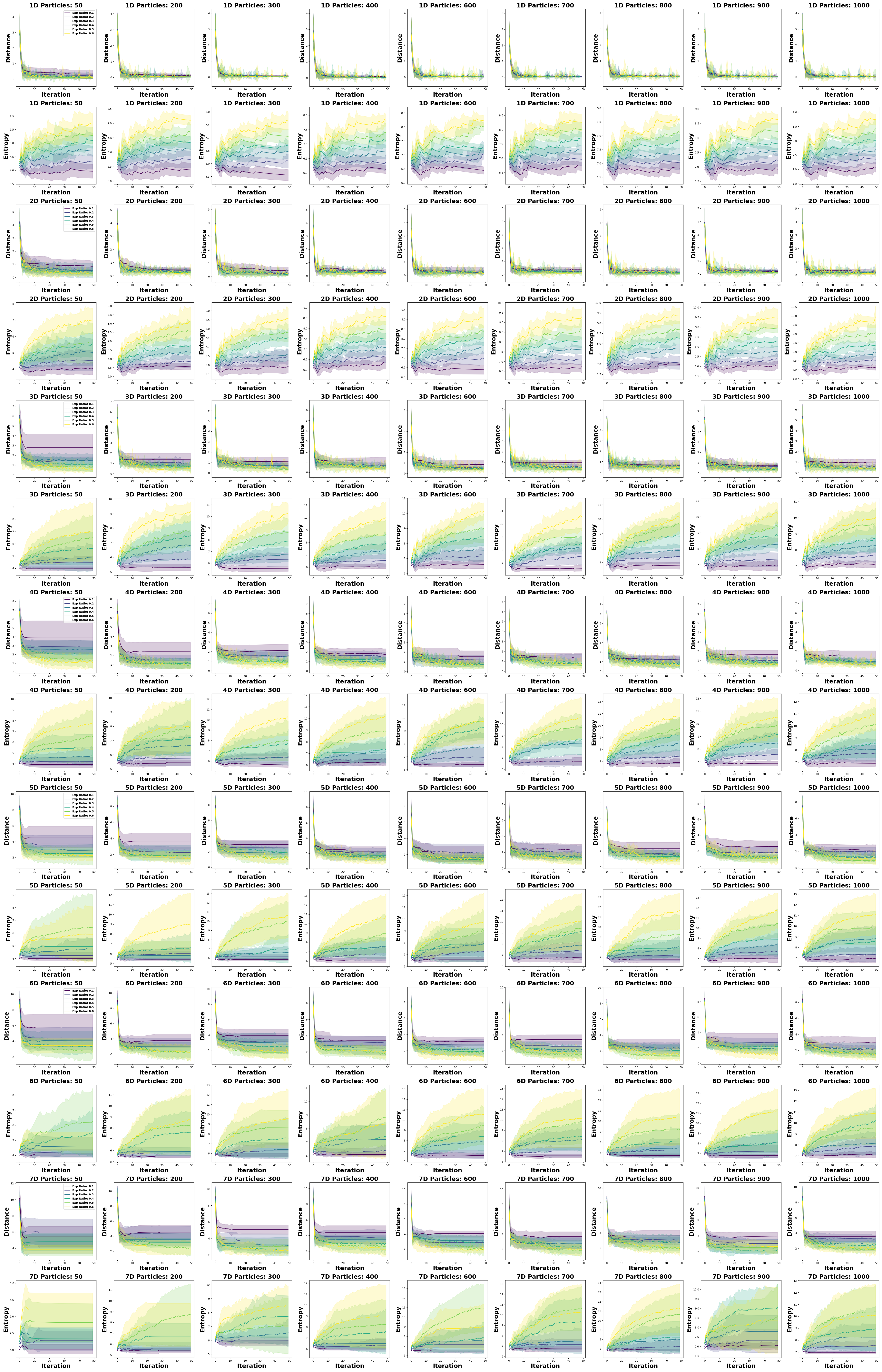}
\caption{Performance Across 1D to 7D Scenarios: Impact of Particle Numbers and Exploration Ratios.}
\label{fig_1-7D}
\end{figure}

\newpage
\section{Phase 2 Experimental Results  Table}
\label{seb:Lab2}

\subsection{Problem Statement}
\textbf{The Source Term Estimation} (STE) problem involves determining the location, release rate, and temporal characteristics of an unknown hazardous emission, such as chemical, biological, or radiological (CBR) pollutants. Due to the \textit{complex and uncertain nature} of atmospheric dispersion—affected by factors like \textit{wind speed, direction, and turbulence}—STE presents a \textit{highly nonlinear and ill-posed} challenge. Traditional methods rely on \textit{static sensor networks and inverse modelling}, but modern approaches integrate \textit{mobile robots, adaptive path planning, Bayesian inference, and particle filtering} to enhance \textit{accuracy, efficiency, and real-time adaptability}. These advancements are crucial for applications in \textit{environmental monitoring, emergency response, and autonomous search operations}.

\subsection{Source Term Estimation: Algorithms and Technical Details}
\textbf{Introduction:}
Source Term Estimation (STE) aims to determine the location, release rate, and temporal characteristics of hazardous emissions such as chemical, biological, or radiological (CBR) pollutants. This problem is highly nonlinear, ill-posed, and affected by uncertainties in environmental conditions and sensor noise. Traditional approaches include inverse modeling and static sensor networks, whereas modern techniques leverage mobile robots, Bayesian inference, and particle filtering.

\textbf{Bayesian Inference for Source Term Estimation:}
Bayesian inference provides a probabilistic framework for estimating source parameters given noisy sensor measurements. The estimation process follows Bayes' theorem:
\begin{equation}
P(\boldsymbol{\theta}_k | \mathbf{z}_{1:k}) = \frac{P(\mathbf{z}_{1:k} | \boldsymbol{\theta}k) P(\boldsymbol{\theta}k)}{P(\mathbf{z}_{1:k})}
\end{equation}
where $\boldsymbol{\theta}k$ represents the source term parameters (e.g., location and release rate) at time step $k$, $\mathbf{z}_{1:k}$ denotes observed sensor data up to time $k$, and $P(\mathbf{z}_{1:k} | \boldsymbol{\theta}_k)$ is the likelihood function derived from an atmospheric dispersion model. Bayesian inference is often implemented using Sequential Monte Carlo methods such as particle filtering.

\textbf{Particle Filter for Probabilistic Estimation}
Particle filtering (Sequential Monte Carlo) is widely used for dynamic state estimation in non-Gaussian, nonlinear systems. It approximates the posterior distribution using a set of weighted particles:
\begin{equation}
P(\boldsymbol{\theta}_k | \mathbf{z}_{1:k}) \approx \sum_{i=1}^{N} w_k^i \delta(\boldsymbol{\theta}_k - \boldsymbol{\theta}_k^i)
\end{equation}
where $\boldsymbol{\theta}_k^i$ represents the $i$-th particle and $w_k^i$ is its weight. 
This method continuously refines source term estimates by incorporating new measurements while directing a mobile robot to informative locations.

\textbf{Information-Based Sensor Planning}
To optimize data collection for source term estimation, an information-based search strategy is used. The objective is to guide the robot to locations that maximize information gain regarding the unknown source. The Kullback-Leibler (KL) divergence is used as a measure of information gain:
\begin{equation}
\mathcal{U}(a_k, \mathbf{z}_{k+1}) = D{KL} \left( P(\boldsymbol{\theta}{k+1} | \mathbf{z}_{1:k}, \mathbf{z}_{k+1}) | P(\boldsymbol{\theta}_{k+1} | \mathbf{z}_{1:k}) \right)
\end{equation}
where $a_k$ is the chosen action at time step $k$. The next measurement location is selected by maximizing the expected utility:
\begin{equation}
a_k^* = \arg\max{a_k} \mathbb{E}[\mathcal{U}(a_k, \mathbf{z}_{k+1})]
\end{equation}
This information-theoretic approach ensures that the robot systematically reduces uncertainty about the source parameters while navigating the environment.

\subsection{Tabel}
\begin{table}[H]
\centering
\caption{Results for Priori\_Scope = 0.3}

{\footnotesize 
\begin{tabular}{|c|c|c|c|c|c|c|c|c|}
\hline
\textbf{Type\_Prior} & \textbf{ratio} & \textbf{Success\_Rate} & \textbf{Entropy} & \textbf{Ent\_Var} & \textbf{Distances} & \textbf{Distances\_Var} & \textbf{Average\_Step} \\ \hline
\multirow{6}{*}{Uniform} & 0.1 & 0.79 & 6.86 & 0.34 & 0.44 & 0.271 & 70.34 \\
 & 0.2 & 0.8 & 6.8 & 0.294 & 0.55 & 0.375 & 65.02 \\
 & 0.3 & 0.81 & 6.79 & 0.344 & 0.6 & 0.622 & 67.29 \\
 & 0.4 & 0.82 & 6.74 & 0.367 & 0.6 & 0.664 & 68.16 \\
 & 0.5 & 0.81 & 6.89 & 0.389 & 0.66 & 0.78 & 69.28 \\
 & 0.6 & 0.77 & 6.79 & 0.393 & 0.7 & 0.931 & 66.92 \\ \hline
\multirow{6}{*}{Beta} & 0.1 & 0.86 & 6.74 & 0.44 & 0.43 & 0.297 & 67.66 \\
 & 0.2 & 0.81 & 6.77 & 0.467 & 0.51 & 0.397 & 66.01 \\
 & 0.3 & 0.81 & 6.87 & 0.407 & 0.42 & 0.315 & 67.36 \\
 & 0.4 & 0.81 & 6.82 & 0.286 & 0.46 & 0.25 & 69.17 \\
 & 0.5 & 0.85 & 6.82 & 0.565 & 0.47 & 0.591 & 65.02 \\
 & 0.6 & 0.81 & 6.85 & 0.417 & 0.6 & 0.574 & 65.59 \\ \hline
\multirow{6}{*}{Gaussian} & 0.1 & 0.85 & 7.27 & 0.471 & 0.2 & 0.03 & 73.94 \\
 & 0.2 & 0.85 & 7.31 & 0.458 & 0.31 & 0.231 & 74.84 \\
 & 0.3 & 0.87 & 7.17 & 0.336 & 0.22 & 0.049 & 71.38 \\
 & 0.4 & 0.82 & 7.32 & 0.392 & 0.24 & 0.071 & 73.23 \\
 & 0.5 & 0.85 & 7.25 & 0.378 & 0.33 & 0.481 & 75.84 \\
 & 0.6 & 0.76 & 7.31 & 0.401 & 0.31 & 0.211 & 76.66 \\ \hline
\multirow{6}{*}{Dirichlet} & 0.1 & 0.76 & 6.41 & 0.274 & 0.28 & 0.057 & 78.34 \\
 & 0.2 & 0.7 & 6.43 & 0.388 & 0.42 & 2.591 & 79.84 \\
 & 0.3 & 0.72 & 6.42 & 0.303 & 0.37 & 0.382 & 80.03 \\
 & 0.4 & 0.69 & 6.38 & 0.306 & 0.58 & 5.001 & 80.16 \\
 & 0.5 & 0.82 & 6.59 & 0.273 & 0.25 & 0.082 & 79.94 \\
 & 0.6 & 0.81 & 6.49 & 0.226 & 0.32 & 0.352 & 77.07 \\ \hline
\end{tabular}}
\end{table}

\begin{table}[H]
\centering
\caption{Detailed Results for Priori\_Scope = 0.3}
{\footnotesize 
\begin{tabular}{|c|c|c|c|c|c|c|c|c|}
\hline
\textbf{Type\_Prior} & \textbf{Ratio} & \textbf{Success\_Rate} & \textbf{Entropy} & \textbf{Ent\_Var} & \textbf{Distances} & \textbf{Distances\_Var} & \textbf{Average\_Step} \\ \hline
\multirow{6}{*}{Star\_Shape} & 0.1 & 0.76 & 6.95 & 0.309 & 0.35 & 0.5 & 73.71 \\
 & 0.2 & 0.83 & 6.87 & 0.343 & 0.27 & 0.229 & 69.92 \\
 & 0.3 & 0.84 & 6.82 & 0.458 & 0.37 & 0.154 & 70.05 \\
 & 0.4 & 0.91 & 6.75 & 0.309 & 0.3 & 0.096 & 64.61 \\
 & 0.5 & 0.78 & 6.77 & 0.334 & 0.32 & 0.139 & 72.96 \\
 & 0.6 & 0.83 & 6.73 & 0.343 & 0.31 & 0.11 & 70.14 \\ \hline
\multirow{6}{*}{${1}/{4}$ Ring} & 0.1 & 0.87 & 6.94 & 0.347 & 0.37 & 0.167 & 65.69 \\
 & 0.2 & 0.84 & 6.93 & 0.339 & 0.36 & 0.259 & 68.64 \\
 & 0.3 & 0.82 & 6.93 & 0.353 & 0.39 & 0.28 & 67.76 \\
 & 0.4 & 0.85 & 6.96 & 0.361 & 0.36 & 0.38 & 68.61 \\
 & 0.5 & 0.78 & 7.02 & 0.347 & 0.5 & 0.697 & 70.7 \\
 & 0.6 & 0.84 & 6.94 & 0.374 & 0.39 & 0.363 & 64.16 \\ \hline
\multirow{6}{*}{${1}/{2}$ Ring} & 0.1 & 0.83 & 6.93 & 0.318 & 0.36 & 0.141 & 69.01 \\
 & 0.2 & 0.85 & 6.97 & 0.311 & 0.42 & 0.382 & 67.82 \\
 & 0.3 & 0.83 & 7.13 & 0.453 & 0.44 & 0.675 & 69.17 \\
 & 0.4 & 0.79 & 6.91 & 0.371 & 0.4 & 0.333 & 66.57 \\
 & 0.5 & 0.78 & 6.85 & 0.368 & 0.44 & 0.386 & 70.25 \\
 & 0.6 & 0.86 & 6.65 & 0.374 & 0.39 & 0.256 & 67.53 \\ \hline
\multirow{6}{*}{${3}/{4}$ Ring} & 0.1 & 0.85 & 7.07 & 0.324 & 0.42 & 0.301 & 68.58 \\
 & 0.2 & 0.81 & 7.08 & 0.204 & 0.3 & 0.116 & 69.45 \\
 & 0.3 & 0.84 & 7.2 & 0.25 & 0.44 & 0.418 & 68.18 \\
 & 0.4 & 0.78 & 7.0 & 0.405 & 0.41 & 0.405 & 68.43 \\
 & 0.5 & 0.83 & 7.0 & 0.378 & 0.38 & 0.209 & 68.77 \\
 & 0.6 & 0.82 & 7.1 & 0.343 & 0.45 & 0.324 & 67.1 \\ \hline
\end{tabular}}
\end{table}

\begin{table}[H]
\centering
\caption{Results for Priori\_Scope = 0.4}
{\footnotesize 
\begin{tabular}{|c|c|c|c|c|c|c|c|}
\hline
\textbf{Type\_Prior} & \textbf{ratio} & \textbf{Success\_Rate} & \textbf{Entropy} & \textbf{Ent\_Var} & \textbf{Distances} & \textbf{Distances\_Var} & \textbf{Average\_Step} \\ \hline
\multirow{6}{*}{Uniform} & 0.1 & 0.95 & 6.5 & 0.3 & 0.41 & 0.23 & 62.94 \\
 & 0.2 & 0.88 & 6.54 & 0.23 & 0.42 & 0.293 & 62.82 \\
 & 0.3 & 0.88 & 6.53 & 0.313 & 0.38 & 0.165 & 61.53 \\
 & 0.4 & 0.88 & 6.58 & 0.218 & 0.35 & 0.131 & 66.46 \\
 & 0.5 & 0.87 & 6.56 & 0.229 & 0.36 & 0.124 & 67.39 \\
 & 0.6 & 0.84 & 6.59 & 0.315 & 0.39 & 0.197 & 66.55 \\ \hline
\multirow{6}{*}{Beta} & 0.1 & 0.8 & 6.97 & 0.454 & 0.46 & 0.538 & 67.26 \\
 & 0.2 & 0.8 & 6.93 & 0.463 & 0.47 & 0.224 & 70.74 \\
 & 0.3 & 0.83 & 6.91 & 0.404 & 0.41 & 0.216 & 69.87 \\
 & 0.4 & 0.88 & 6.8 & 0.334 & 0.45 & 0.572 & 64.32 \\
 & 0.5 & 0.79 & 6.88 & 0.429 & 0.47 & 0.401 & 67.21 \\
 & 0.6 & 0.84 & 6.84 & 0.346 & 0.46 & 0.267 & 66.54 \\ \hline
\multirow{6}{*}{Gaussian} & 0.1 & 0.83 & 7.15 & 0.302 & 0.21 & 0.033 & 70.95 \\
 & 0.2 & 0.69 & 7.25 & 0.386 & 0.3 & 0.249 & 79.14 \\
 & 0.3 & 0.83 & 7.21 & 0.412 & 0.23 & 0.047 & 75.53 \\
 & 0.4 & 0.85 & 7.29 & 0.353 & 0.31 & 0.189 & 76.0 \\
 & 0.5 & 0.8 & 7.13 & 0.463 & 0.28 & 0.108 & 74.35 \\
 & 0.6 & 0.85 & 7.12 & 0.282 & 0.24 & 0.073 & 70.34 \\ \hline
\multirow{6}{*}{Dirichlet} & 0.1 & 0.76 & 6.43 & 0.307 & 0.41 & 2.58 & 77.73 \\
 & 0.2 & 0.74 & 6.54 & 0.317 & 0.25 & 0.034 & 80.34 \\
 & 0.3 & 0.82 & 6.47 & 0.328 & 0.49 & 2.626 & 73.88 \\
 & 0.4 & 0.72 & 6.55 & 0.358 & 0.47 & 2.283 & 81.13 \\
 & 0.5 & 0.75 & 6.51 & 0.354 & 0.28 & 0.107 & 76.86 \\
 & 0.6 & 0.78 & 6.5 & 0.237 & 0.38 & 2.305 & 78.44 \\ \hline
\end{tabular}}
\end{table}

\begin{table}[H]
\centering
\caption{Detailed Results for Priori\_Scope = 0.4}
{\footnotesize 
\begin{tabular}{|c|c|c|c|c|c|c|c|c|}
\hline
\textbf{Type\_Prior} & \textbf{Ratio} & \textbf{Success\_Rate} & \textbf{Entropy} & \textbf{Ent\_Var} & \textbf{Distances} & \textbf{Distances\_Var} & \textbf{Average\_Step} \\ \hline
\multirow{6}{*}{Star\_Shape} & 0.1 & 0.81 & 6.83 & 0.379 & 0.35 & 0.461 & 70.92 \\
 & 0.2 & 0.83 & 6.8 & 0.261 & 0.29 & 0.312 & 69.76 \\
 & 0.3 & 0.82 & 6.79 & 0.479 & 0.49 & 0.727 & 71.24 \\
 & 0.4 & 0.89 & 6.79 & 0.415 & 0.28 & 0.121 & 65.88 \\
 & 0.5 & 0.85 & 6.78 & 0.332 & 0.34 & 0.127 & 72.92 \\
 & 0.6 & 0.81 & 6.85 & 0.391 & 0.37 & 0.189 & 70.59 \\ \hline
\multirow{6}{*}{${1}/{4}$ Ring} & 0.1 & 0.77 & 6.94 & 0.422 & 0.49 & 0.343 & 68.5 \\
 & 0.2 & 0.79 & 6.98 & 0.428 & 0.46 & 0.225 & 67.29 \\
 & 0.3 & 0.8 & 7 & 0.4 & 0.49 & 0.287 & 68.32 \\
 & 0.4 & 0.85 & 6.88 & 0.443 & 0.6 & 0.906 & 64.3 \\
 & 0.5 & 0.85 & 6.89 & 0.342 & 0.51 & 0.505 & 66.42 \\
 & 0.6 & 0.89 & 6.85 & 0.272 & 0.46 & 0.288 & 64.62 \\ \hline
\multirow{6}{*}{${1}/{2}$ Ring} & 0.1 & 0.83 & 6.85 & 0.345 & 0.43 & 0.197 & 63.65 \\
 & 0.2 & 0.85 & 7.04 & 0.284 & 0.35 & 0.127 & 67.64 \\
 & 0.3 & 0.83 & 6.88 & 0.323 & 0.47 & 0.274 & 72.03 \\
 & 0.4 & 0.78 & 6.98 & 0.374 & 0.41 & 0.197 & 69.44 \\
 & 0.5 & 0.85 & 6.81 & 0.417 & 0.68 & 0.974 & 61.74 \\
 & 0.6 & 0.83 & 6.98 & 0.315 & 0.45 & 0.354 & 66.39 \\ \hline
\multirow{6}{*}{${3}/{4}$ Ring} & 0.1 & 0.85 & 7.01 & 0.252 & 0.44 & 0.316 & 69.24 \\
 & 0.2 & 0.84 & 6.95 & 0.352 & 0.49 & 0.303 & 65.88 \\
 & 0.3 & 0.81 & 6.92 & 0.309 & 0.42 & 0.199 & 68.96 \\
 & 0.4 & 0.8 & 6.98 & 0.314 & 0.38 & 0.133 & 70.27 \\
 & 0.5 & 0.75 & 7.12 & 0.223 & 0.42 & 0.185 & 72.34 \\
 & 0.6 & 0.72 & 6.95 & 0.32 & 0.53 & 0.53 & 70.32 \\ \hline
\end{tabular}}
\end{table}

\begin{table}[H]
\centering
\caption{Results for Priori\_Scope = 0.5}
{\footnotesize 
\begin{tabular}{|c|c|c|c|c|c|c|c|c|}
\hline
\textbf{Type\_Prior} & \textbf{Ratio} & \textbf{Success\_Rate} & \textbf{Entropy} & \textbf{Ent\_Var} & \textbf{Distances} & \textbf{Distances\_Var} & \textbf{Average\_Step} \\ \hline
\multirow{6}{*}{Uniform} & 0.1 & 0.95 & 6.41 & 0.178 & 0.18 & 0.019 & 65.15 \\
 & 0.2 & 0.92 & 6.38 & 0.191 & 0.21 & 0.072 & 64.23 \\
 & 0.3 & 0.89 & 6.29 & 0.17 & 0.18 & 0.015 & 68.81 \\
 & 0.4 & 0.92 & 6.43 & 0.208 & 0.18 & 0.018 & 69.31 \\
 & 0.5 & 0.93 & 6.36 & 0.208 & 0.2 & 0.029 & 64.48 \\
 & 0.6 & 0.95 & 6.38 & 0.211 & 0.16 & 0.011 & 66.8 \\ \hline
\multirow{6}{*}{Beta} & 0.1 & 0.85 & 6.84 & 0.338 & 0.4 & 0.341 & 66.36 \\
 & 0.2 & 0.75 & 6.83 & 0.379 & 0.59 & 0.74 & 69.1 \\
 & 0.3 & 0.72 & 6.8 & 0.239 & 0.49 & 0.306 & 69.22 \\
 & 0.4 & 0.76 & 6.74 & 0.44 & 0.53 & 0.848 & 67.82 \\
 & 0.5 & 0.83 & 6.87 & 0.44 & 0.55 & 0.481 & 65.84 \\
 & 0.6 & 0.77 & 6.87 & 0.277 & 0.43 & 0.256 & 70.24 \\ \hline
\multirow{6}{*}{Gaussian} & 0.1 & 0.81 & 7.22 & 0.434 & 0.28 & 0.066 & 73.81 \\
 & 0.2 & 0.8 & 7.21 & 0.419 & 0.27 & 0.144 & 75.17 \\
 & 0.3 & 0.8 & 7.25 & 0.42 & 0.23 & 0.057 & 75.71 \\
 & 0.4 & 0.82 & 7.26 & 0.596 & 0.29 & 0.17 & 73.32 \\
 & 0.5 & 0.78 & 7.21 & 0.395 & 0.32 & 0.497 & 78.23 \\
 & 0.6 & 0.78 & 7.17 & 0.338 & 0.23 & 0.082 & 75.97 \\ \hline
\multirow{6}{*}{Dirichlet} & 0.1 & 0.74 & 6.46 & 0.364 & 0.46 & 2.724 & 78.38 \\
 & 0.2 & 0.74 & 6.48 & 0.34 & 0.47 & 3.023 & 76.0 \\
 & 0.3 & 0.78 & 6.44 & 0.225 & 0.28 & 0.138 & 80.22 \\
 & 0.4 & 0.76 & 6.48 & 0.22 & 0.29 & 0.122 & 78.8 \\
 & 0.5 & 0.81 & 6.43 & 0.266 & 0.48 & 3.818 & 79.33 \\
 & 0.6 & 0.69 & 6.32 & 0.241 & 0.48 & 2.343 & 78.71 \\ \hline
\end{tabular}}
\end{table}

\begin{table}[H]
\centering
\caption{Detailed Results for Priori\_Scope = 0.5}
{\footnotesize 
\begin{tabular}{|c|c|c|c|c|c|c|c|c|}
\hline
\textbf{Type\_Prior} & \textbf{Ratio} & \textbf{Success\_Rate} & \textbf{Entropy} & \textbf{Ent\_Var} & \textbf{Distances} & \textbf{Distances\_Var} & \textbf{Average\_Step} \\ \hline
\multirow{6}{*}{Star\_Shape} & 0.1 & 0.79 & 6.84 & 0.35 & 0.35 & 0.259 & 70.27 \\
 & 0.2 & 0.8 & 6.77 & .031 & 0.38 & 0.342 & 70.53 \\
 & 0.3 & 0.8 & 6.8 & 0.29 & 0.28 & 0.061 & 71.47 \\
 & 0.4 & 0.77 & 6.84 & 0.44 & 0.35 & 0.221 & 73.73 \\
 & 0.5 & 0.81 & 6.81 & 0.42 & 0.31 & 0.121 & 72.0 \\
 & 0.6 & 0.79 & 6.78 & 0.301 & 0.37 & 0.435 & 72.73 \\ \hline
\multirow{6}{*}{${1}/{4}$ Ring} & 0.1 & 0.83 & 6.66 & 0.289 & 0.52 & 0.411 & 65.64 \\
 & 0.2 & 0.88 & 6.61 & 0.303 & 0.37 & 0.214 & 64.39 \\
 & 0.3 & 0.86 & 6.66 & 0.365 & 0.49 & 0.337 & 64.92 \\
 & 0.4 & 0.8 & 6.7 & 0.353 & 0.41 & 0.206 & 65.6 \\
 & 0.5 & 0.81 & 6.66 & 0.312 & 0.46 & 0.435 & 66.15 \\
 & 0.6 & 0.86 & 6.67 & 0.273 & 0.52 & 0.449 & 62.91 \\ \hline
\multirow{6}{*}{${1}/{2}$ Ring} & 0.1 & 0.85 & 6.69 & 0.265 & 0.37 & 0.255 & 66.23 \\
 & 0.2 & 0.88 & 6.74 & 0.238 & 0.39 & 0.236 & 66.79 \\
 & 0.3 & 0.85 & 6.65 & 0.277 & 0.4 & 0.144 & 67.12 \\
 & 0.4 & 0.89 & 6.6 & 0.232 & 0.39 & 0.275 & 60.08 \\
 & 0.5 & 0.88 & 6.71 & 0.26 & 0.44 & 0.312 & 60.99 \\
 & 0.6 & 0.84 & 6.71 & 0.261 & 0.54 & 0.567 & 64.59 \\ \hline
\multirow{6}{*}{${3}/{4}$ Ring} & 0.1 & 0.85 & 6.68 & 0.348 & 0.32 & 0.266 & 69.37 \\
 & 0.2 & 0.87 & 6.65 & 0.204 & 0.44 & 0.457 & 63.99 \\
 & 0.3 & 0.85 & 6.77 & 0.262 & 0.41 & 0.372 & 66.52 \\
 & 0.4 & 0.85 & 7.02 & 0.436 & 0.32 & 0.077 & 65.84 \\
 & 0.5 & 0.81 & 6.94 & 0.347 & 0.39 & 0.225 & 68.49 \\
 & 0.6 & 0.87 & 7.16 & 0.423 & 0.34 & 0.177 & 65.16 \\ \hline
\end{tabular}}
\end{table}

\begin{table}[H]
\centering
\caption{Results for Priori\_Scope = 0.6}
{\footnotesize 
\begin{tabular}{|c|c|c|c|c|c|c|c|}
\hline
\textbf{Type\_Prior} & \textbf{Ratio} & \textbf{Success\_Rate} & \textbf{Entropy} & \textbf{Ent\_Var} & \textbf{Distances} & \textbf{Distances\_Var} & \textbf{Average\_Step} \\ \hline
\multirow{6}{*}{Uniform} & 0.1 & 0.89 & 6.32 & 0.157 & 0.18 & 0.028 & 68.61 \\
 & 0.2 & 0.9 & 6.3 & 0.197 & 0.18 & 0.028 & 68.3 \\
 & 0.3 & 0.85 & 6.3 & 0.138 & 0.18 & 0.027 & 71.54 \\
 & 0.4 & 0.91 & 6.38 & 0.129 & 0.15 & 0.011 & 69.55 \\
 & 0.5 & 0.89 & 6.34 & 0.124 & 0.16 & 0.022 & 70.83 \\
 & 0.6 & 0.89 & 6.34 & 0.173 & 0.17 & 0.014 & 71.27 \\ \hline
\multirow{6}{*}{Beta} & 0.1 & 0.89 & 6.9 & 0.505 & 0.56 & 0.988 & 63.94 \\
 & 0.2 & 0.8 & 6.84 & 0.368 & 0.68 & 0.959 & 66.61 \\
 & 0.3 & 0.82 & 6.88 & 0.41 & 0.54 & 0.57 & 66.47 \\
 & 0.4 & 0.82 & 6.91 & 0.263 & 0.37 & 0.19 & 70.78 \\
 & 0.5 & 0.85 & 6.82 & 0.312 & 0.44 & 0.254 & 67.76 \\
 & 0.6 & 0.85 & 6.87 & 0.358 & 0.4 & 0.206 & 66.45 \\ \hline
\multirow{6}{*}{Gaussian} & 0.1 & 0.79 & 7.31 & 0.343 & 0.23 & 0.062 & 77.15 \\
 & 0.2 & 0.82 & 7.21 & 0.434 & 0.27 & 0.069 & 74.39 \\
 & 0.3 & 0.79 & 7.31 & 0.497 & 0.27 & 0.158 & 74.55 \\
 & 0.4 & 0.82 & 7.3 & 0.401 & 0.24 & 0.036 & 74.71 \\
 & 0.5 & 0.79 & 7.33 & 0.47 & 0.28 & 0.178 & 75.2 \\
 & 0.6 & 0.8 & 7.36 & 0.45 & 0.23 & 0.044 & 75.22 \\ \hline
\multirow{6}{*}{Dirichlet} & 0.1 & 0.77 & 6.46 & 0.356 & 0.28 & 0.09 & 79.9 \\
 & 0.2 & 0.76 & 6.46 & 0.323 & 0.32 & 0.227 & 81.13 \\
 & 0.3 & 0.67 & 6.51 & 0.231 & 0.28 & 0.075 & 82.06 \\
 & 0.4 & 0.8 & 6.43 & 0.409 & 0.56 & 4.825 & 76.11 \\
 & 0.5 & 0.72 & 6.48 & 0.334 & 0.53 & 3.145 & 78.78 \\
 & 0.6 & 0.77 & 6.49 & 0.222 & 0.31 & 0.099 & 76.47 \\ \hline
\end{tabular}}
\end{table}

\begin{table}[H]
\centering
\caption{Detailed Results for Priori\_Scope = 0.6}
{\footnotesize 
\begin{tabular}{|c|c|c|c|c|c|c|c|c|}
\hline
\textbf{Type\_Prior} & \textbf{Ratio} & \textbf{Success\_Rate} & \textbf{Entropy} & \textbf{Ent\_Var} & \textbf{Distances} & \textbf{Distances\_Var} & \textbf{Average\_Step} \\ \hline
\multirow{6}{*}{Star\_Shape} & 0.1 & 0.78 & 6.75 & 0.46 & 0.34 & 0.146 & 72.34 \\
 & 0.2 & 0.75 & 6.87 & 0.37 & 0.43 & 0.392 & 69.93 \\
 & 0.3 & 0.84 & 6.76 & 0.32 & 0.43 & 0.574 & 67.92 \\
 & 0.4 & 0.82 & 6.86 & 0.451 & 0.36 & 0.193 & 69.84 \\
 & 0.5 & 0.88 & 6.96 & 0.586 & 0.34 & 0.339 & 69.78 \\
 & 0.6 & 0.81 & 6.84 & 0.337 & 0.36 & 0.996 & 72.07 \\ \hline
\multirow{6}{*}{${1}/{4}$ Ring} & 0.1 & 0.9 & 6.51 & 0.2 & 0.24 & 0.076 & 66.85 \\
 & 0.2 & 0.89 & 6.43 & 0.195 & 0.24 & 0.089 & 62.62 \\
 & 0.3 & 0.89 & 6.4 & 0.21 & 0.29 & 0.104 & 64.95 \\
 & 0.4 & 0.87 & 6.49 & 0.163 & 0.28 & 0.085 & 68.09 \\
 & 0.5 & 0.9 & 6.4 & 0.192 & 0.22 & 0.046 & 63.51 \\
 & 0.6 & 0.91 & 6.46 & 0.179 & 0.27 & 0.076 & 64.09 \\ \hline
\multirow{6}{*}{${1}/{2}$ Ring} & 0.1 & 0.89 & 6.56 & 0.258 & 0.21 & 0.062 & 65.95 \\
 & 0.2 & 0.95 & 6.53 & 0.192 & 0.2 & 0.035 & 65.21 \\
 & 0.3 & 0.93 & 6.49 & 0.194 & 0.27 & 0.133 & 62.98 \\
 & 0.4 & 0.89 & 6.49 & 0.198 & 0.21 & 0.033 & 63.4 \\
 & 0.5 & 0.91 & 6.4 & 0.173 & 0.29 & 0.194 & 63.24 \\
 & 0.6 & 0.92 & 6.51 & 0.168 & 0.29 & 0.204 & 65.6 \\ \hline
\multirow{6}{*}{${3}/{4}$ Ring} & 0.1 & 0.95 & 6.51 & 0.211 & 0.21 & 0.046 & 65.06 \\
 & 0.2 & 0.88 & 6.41 & 0.163 & 0.27 & 0.085 & 66.85 \\
 & 0.3 & 0.9 & 6.48 & 0.157 & 0.25 & 0.062 & 64.44 \\
 & 0.4 & 0.85 & 6.46 & 0.183 & 0.26 & 0.077 & 65.81 \\
 & 0.5 & 0.9 & 6.55 & 0.168 & 0.25 & 0.173 & 63.4 \\
 & 0.6 & 0.89 & 6.52 & 0.174 & 0.29 & 0.115 & 67.42 \\ \hline
\end{tabular}}
\end{table}

\begin{table}[H]
\centering
\caption{Results for Priori\_Scope = 0.7}
{\footnotesize 
\begin{tabular}{|c|c|c|c|c|c|c|c|}
\hline
\textbf{Type\_Prior} & \textbf{Ratio} & \textbf{Success\_Rate} & \textbf{Entropy} & \textbf{Ent\_Var} & \textbf{Distances} & \textbf{Distances\_Var} & \textbf{Average\_Step} \\ \hline
\multirow{6}{*}{Uniform} & 0.1 & 0.94 & 6.29 & 0.151 & 0.16 & 0.008 & 69.16 \\
 & 0.2 & 0.94 & 6.25 & 0.144 & 0.14 & 0.005 & 68.65 \\
 & 0.3 & 0.97 & 6.35 & 0.191 & 0.15 & 0.009 & 71.11 \\
 & 0.4 & 0.91 & 6.38 & 0.155 & 0.15 & 0.007 & 71.99 \\
 & 0.5 & 0.92 & 6.31 & 0.153 & 0.15 & 0.011 & 68.58 \\
 & 0.6 & 0.91 & 6.35 & 0.121 & 0.16 & 0.011 & 70.94 \\ \hline
\multirow{6}{*}{Beta} & 0.1 & 0.81 & 6.81 & 0.27 & 0.54 & 0.613 & 66.03 \\
 & 0.2 & 0.84 & 6.7 & 0.401 & 0.44 & 0.317 & 63.76 \\
 & 0.3 & 0.88 & 6.91 & 0.491 & 0.53 & 0.582 & 64.15 \\
 & 0.4 & 0.81 & 6.83 & 0.334 & 0.6 & 0.646 & 67.63 \\
 & 0.5 & 0.86 & 6.82 & 0.45 & 0.4 & 0.191 & 67.09 \\
 & 0.6 & 0.81 & 6.8 & 0.32 & 0.39 & 0.152 & 69.29 \\ \hline
\multirow{6}{*}{Gaussian} & 0.1 & 0.79 & 7.24 & 0.317 & 0.26 & 0.099 & 76.47 \\
 & 0.2 & 0.85 & 7.16 & 0.319 & 0.29 & 0.286 & 75.74 \\
 & 0.3 & 0.75 & 7.21 & 0.289 & 0.25 & 0.051 & 75.87 \\
 & 0.4 & 0.84 & 7.3 & 0.363 & 0.25 & 0.08 & 72.22 \\
 & 0.5 & 0.81 & 7.28 & 0.381 & 0.36 & 0.625 & 75.77 \\
 & 0.6 & 0.8 & 7.25 & 0.439 & 0.31 & 0.112 & 75.61 \\ \hline
\multirow{6}{*}{Dirichlet} & 0.1 & 0.77 & 6.42 & 0.263 & 0.48 & 2.61 & 77.84 \\
 & 0.2 & 0.81 & 6.49 & 0.285 & 0.48 & 3.412 & 77.77 \\
 & 0.3 & 0.83 & 6.48 & 0.263 & 0.3 & 0.105 & 76.44 \\
 & 0.4 & 0.82 & 6.46 & 0.21 & 0.29 & 0.134 & 78.39 \\
 & 0.5 & 0.83 & 6.5 & 0.341 & 0.65 & 7.065 & 77.58 \\
 & 0.6 & 0.77 & 6.48 & 0.324 & 0.46 & 3.102 & 81.27 \\ \hline
\end{tabular}}
\end{table}

\begin{table}[H]
\centering
\caption{Detailed Results for Priori\_Scope = 0.7}
{\footnotesize 
\begin{tabular}{|c|c|c|c|c|c|c|c|c|}
\hline
\textbf{Type\_Prior} & \textbf{Ratio} & \textbf{Success\_Rate} & \textbf{Entropy} & \textbf{Ent\_Var} & \textbf{Distances} & \textbf{Distances\_Var} & \textbf{Average\_Step} \\ \hline
\multirow{6}{*}{Star\_Shape} & 0.1 & 0.83 & 6.83 & 0.484 & 0.3 & 0.089 & 71.16 \\
 & 0.2 & 0.75 & 6.74 & 0.29 & 0.49 & 0.771 & 77.0 \\
 & 0.3 & 0.84 & 6.77 & 0.36 & 0.28 & 0.109 & 69.88 \\
 & 0.4 & 0.86 & 6.71 & 0.34 & 0.34 & 0.156 & 67.53 \\
 & 0.5 & 0.8 & 6.81 & 0.292 & 0.41 & 0.439 & 73.43 \\
 & 0.6 & 0.81 & 6.76 & 0.44 & 0.34 & 0.093 & 71.75 \\ \hline
\multirow{6}{*}{${1}/{4}$ Ring} & 0.1 & 0.94 & 6.33 & 0.192 & 0.16 & 0.007 & 67.88 \\
 & 0.2 & 0.94 & 6.36 & 0.16 & 0.17 & 0.021 & 68.01 \\
 & 0.3 & 0.94 & 6.38 & 0.217 & 0.16 & 0.011 & 69.73 \\
 & 0.4 & 0.89 & 6.3 & 0.217 & 0.18 & 0.021 & 68.55 \\
 & 0.5 & 0.91 & 6.32 & 0.212 & 0.22 & 0.064 & 68.87 \\
 & 0.6 & 0.97 & 6.36 & 0.172 & 0.17 & 0.012 & 66.22 \\ \hline
\multirow{6}{*}{${1}/{2}$ Ring} & 0.1 & 0.91 & 6.39 & 0.156 & 0.19 & 0.028 & 68.19 \\
 & 0.2 & 0.96 & 6.36 & 0.173 & 0.17 & 0.02 & 64.97 \\
 & 0.3 & 0.97 & 6.46 & 0.138 & 0.14 & 0.01 & 66.89 \\
 & 0.4 & 0.94 & 6.31 & 0.171 & 0.18 & 0.055 & 64.17 \\
 & 0.5 & 0.87 & 6.41 & 0.16 & 0.21 & 0.084 & 68.63 \\
 & 0.6 & 0.88 & 6.33 & 0.188 & 0.21 & 0.053 & 66.43 \\ \hline
\multirow{6}{*}{${3}/{4}$ Ring} & 0.1 & 0.96 & 6.44 & 0.198 & 0.15 & 0.005 & 68.0 \\
 & 0.2 & 0.96 & 6.32 & 0.161 & 0.17 & 0.019 & 66.35 \\
 & 0.3 & 0.86 & 6.31 & 0.175 & 0.21 & 0.042 & 69.14 \\
 & 0.4 & 0.92 & 6.29 & 0.178 & 0.2 & 0.034 & 64.52 \\
 & 0.5 & 0.92 & 6.32 & 0.185 & 0.2 & 0.054 & 67.28 \\
 & 0.6 & 0.92 & 6.36 & 0.152 & 0.19 & 0.044 & 68.92 \\ \hline
\end{tabular}}
\end{table}

\begin{table}[H]
\centering
\caption{Results for Priori\_Scope = 0.8}
{\footnotesize 
\begin{tabular}{|c|c|c|c|c|c|c|c|}
\hline
\textbf{Type\_Prior} & \textbf{Ratio} & \textbf{Success\_Rate} & \textbf{Entropy} & \textbf{Ent\_Var} & \textbf{Distances} & \textbf{Distances\_Var} & \textbf{Average\_Step} \\ \hline
\multirow{6}{*}{Uniform} & 0.1 & 0.95 & 6.33 & 0.233 & 0.14 & 0.005 & 67.33 \\
 & 0.2 & 0.91 & 6.3 & 0.149 & 0.15 & 0.009 & 69.29 \\
 & 0.3 & 0.92 & 6.3 & 0.149 & 0.14 & 0.009 & 68.96 \\
 & 0.4 & 0.91 & 6.26 & 0.152 & 0.17 & 0.025 & 69.59 \\
 & 0.5 & 0.88 & 6.28 & 0.149 & 0.16 & 0.008 & 69.27 \\
 & 0.6 & 0.88 & 6.36 & 0.157 & 0.18 & 0.015 & 72.61 \\ \hline
\multirow{6}{*}{Beta} & 0.1 & 0.85 & 6.81 & 0.482 & 0.51 & 0.311 & 66.19 \\
 & 0.2 & 0.84 & 6.73 & 0.51 & 0.45 & 0.396 & 67.58 \\
 & 0.3 & 0.85 & 6.83 & 0.503 & 0.55 & 0.906 & 65.47 \\
 & 0.4 & 0.79 & 6.9 & 0.312 & 0.57 & 0.907 & 71.11 \\
 & 0.5 & 0.81 & 6.92 & 0.325 & 0.56 & 0.604 & 67.71 \\
 & 0.6 & 0.82 & 6.83 & 0.427 & 0.55 & 0.646 & 66.94 \\ \hline
\multirow{6}{*}{Gaussian} & 0.1 & 0.76 & 7.19 & 0.337 & 0.27 & 0.099 & 75.01 \\
 & 0.2 & 0.89 & 7.16 & 0.428 & 0.23 & 0.082 & 71.05 \\
 & 0.3 & 0.78 & 7.11 & 0.387 & 0.24 & 0.051 & 76.52 \\
 & 0.4 & 0.79 & 7.33 & 0.377 & 0.24 & 0.041 & 76.11 \\
 & 0.5 & 0.78 & 7.4 & 0.433 & 0.25 & 0.076 & 75.25 \\
 & 0.6 & 0.71 & 7.33 & 0.424 & 0.28 & 0.106 & 79.39 \\ \hline
\multirow{6}{*}{Dirichlet} & 0.1 & 0.75 & 6.42 & 0.312 & 0.48 & 2.819 & 78.19 \\
 & 0.2 & 0.78 & 6.45 & 0.246 & 0.25 & 0.075 & 79.78 \\
 & 0.3 & 0.76 & 6.44 & 0.23 & 0.33 & 0.434 & 79.91 \\
 & 0.4 & 0.82 & 6.48 & 0.246 & 0.28 & 0.19 & 77.26 \\
 & 0.5 & 0.76 & 6.53 & 0.337 & 0.27 & 0.074 & 77.76 \\
 & 0.6 & 0.79 & 6.5 & 0.43 & 0.73 & 7.302 & 77.64 \\ \hline
\end{tabular}}
\end{table}

\begin{table}[H]
\centering
\caption{Detailed Results for Priori\_Scope = 0.8}
{\footnotesize 
\begin{tabular}{|c|c|c|c|c|c|c|c|c|}
\hline
\textbf{Type\_Prior} & \textbf{Ratio} & \textbf{Success\_Rate} & \textbf{Entropy} & \textbf{Ent\_Var} & \textbf{Distances} & \textbf{Distances\_Var} & \textbf{Average\_Step} \\ \hline
\multirow{6}{*}{Star\_Shape} & 0.1 & 0.81 & 6.82 & 0.316 & 0.37 & 0.208 & 72.14 \\
 & 0.2 & 0.8 & 6.77 & 0.356 & 0.44 & 0.348 & 70.29 \\
 & 0.3 & 0.77 & 6.77 & 0.241 & 0.42 & 0.517 & 73.87 \\
 & 0.4 & 0.84 & 6.89 & 0.321 & 0.35 & 0.168 & 71.15 \\
 & 0.5 & 0.83 & 6.81 & 0.261 & 0.32 & 0.093 & 71.47 \\
 & 0.6 & 0.74 & 6.85 & 0.316 & 0.4 & 0.307 & 72.18 \\ \hline
\multirow{6}{*}{${1}/{4}$ Ring} & 0.1 & 0.91 & 6.23 & 0.162 & 0.19 & 0.133 & 68.05 \\
 & 0.2 & 0.92 & 6.28 & 0.129 & 0.15 & 0.009 & 66.93 \\
 & 0.3 & 0.85 & 6.27 & 0.154 & 0.18 & 0.026 & 71.4 \\
 & 0.4 & 0.95 & 6.33 & 0.161 & 0.19 & 0.059 & 69.54 \\
 & 0.5 & 0.94 & 6.22 & 0.19 & 0.16 & 0.01 & 69.51 \\
 & 0.6 & 0.9 & 6.26 & 0.216 & 0.16 & 0.009 & 67.75 \\ \hline
\multirow{6}{*}{${1}/{2}$ Ring} & 0.1 & 0.84 & 6.32 & 0.161 & 0.19 & 0.073 & 71.6 \\
 & 0.2 & 0.92 & 6.19 & 0.177 & 0.17 & 0.013 & 65.81 \\
 & 0.3 & 0.95 & 6.32 & 0.166 & 0.15 & 0.007 & 67.92 \\
 & 0.4 & 0.91 & 6.26 & 0.222 & 0.15 & 0.009 & 69.37 \\
 & 0.5 & 0.91 & 6.31 & 0.137 & 0.19 & 0.031 & 70.69 \\
 & 0.6 & 0.93 & 6.32 & 0.148 & 0.17 & 0.013 & 69.42 \\ \hline
\multirow{6}{*}{${3}/{4}$ Ring} & 0.1 & 0.92 & 6.28 & 0.189 & 0.18 & 0.036 & 67.55 \\
 & 0.2 & 0.92 & 6.26 & 0.176 & 0.19 & 0.037 & 69.03 \\
 & 0.3 & 0.94 & 6.32 & 0.204 & 0.16 & 0.011 & 70.63 \\
 & 0.4 & 0.94 & 6.28 & 0.159 & 0.2 & 0.058 & 67.52 \\
 & 0.5 & 0.94 & 6.27 & 0.161 & 0.18 & 0.015 & 66.96 \\
 & 0.6 & 0.92 & 6.3 & 0.177 & 0.19 & 0.066 & 68.91 \\ \hline
\end{tabular}}
\end{table}

\end{document}